\newif\ifarxiv

\arxivtrue

\ifarxiv
\documentclass{article}
\usepackage{arxiv}
\usepackage{amsthm}
\else
\documentclass[twoside]{article}
\usepackage{aistats2024}
\fi

\usepackage[utf8]{inputenc} 
\usepackage[T1]{fontenc}    
\usepackage{hyperref}       
\usepackage{url}            
\usepackage{booktabs}       
\usepackage{amsfonts}       
\usepackage{nicefrac}       
\usepackage{microtype}      
\usepackage{xcolor}         

\usepackage{microtype}
\usepackage{graphicx}
\usepackage{booktabs} 
\usepackage{hyperref}

\usepackage{amsmath}
\usepackage{amssymb}
\usepackage{mathtools}
\usepackage{amsthm}

\usepackage[capitalize,noabbrev]{cleveref}
\crefformat{footnote}{#2\footnotemark[#1]#3}
\usepackage[utf8]{inputenc} 
\usepackage[T1]{fontenc}    
\usepackage{hyperref}       
\usepackage{url}            
\usepackage{booktabs}       
\usepackage{amsfonts}       
\usepackage{nicefrac}       
\usepackage{microtype}      
\usepackage{xcolor}         

\usepackage{comment} 
\usepackage{lipsum} 
\usepackage{graphicx}
\usepackage{makecell}
\usepackage{bbm}
\usepackage{epsfig, latexsym, amsfonts, amssymb, graphicx}
\usepackage{soul}
\usepackage{tikz} 
\usepackage{enumitem}
\usetikzlibrary{positioning}
\usepackage{todonotes}
\usepackage{comment}
\usepackage[american]{circuitikz}
\usepackage{subcaption}
\usepackage[algo2e, norelsize, linesnumbered, ruled, lined, boxed, commentsnumbered]{algorithm2e}
\usepackage{times, fancyhdr,xcolor, algorithm,  eso-pic, forloop, url}
\usepackage[numbers]{natbib}

\usepackage[noend]{algpseudocode}
\usetikzlibrary{shapes,arrows}

\DeclareMathOperator*{\argmin}{arg\,min}
\newcommand{\R}{\mathbb{R}}
\newcommand{\x}{\theta}
\newcommand{\z}{\mathbf{z}}
\newcommand{\ahat}{\hat{\mathbf{a}}}
\newcommand{\y}{\theta'}

\newcommand{\dfdx}{\nabla f(\x)}
\newcommand{\dfdy}{\nabla f(\y)}
\newcommand{\dfdxt}{\nabla f(\xt)}
\newcommand{\dxdt}{\dot{\x}}
\newcommand{\xt}{\x(t)}
\newcommand{\xact}{\x^{(a)}}
\newcommand{\xacti}{\x_i^{(a)}}
\newcommand{\Ab}{\mathbf{A}}

\usepackage{accents}

\usepackage[export]{adjustbox}
\algdef{SE}[DOWHILE]{Do}{doWhile}{\algorithmicdo}[1]{\algorithmicwhile\ #1}%
\renewcommand{\cite}[1]{\citep{#1}}
\newtheorem{theorem}{Theorem}
\newtheorem{lemma}{Lemma}

\newtheorem{definition}{Definition}

\pdfoutput=1
\title{Towards Hyperparameter-Agnostic DNN Training via Dynamical System Insights}

\ifarxiv
\author{%
  Carmel Fiscko\thanks{These authors contributed equally.} \\
  Carnegie Mellon University\\
  Pittsburgh, PA 15213 \\
  \texttt{cfiscko@andrew.cmu.edu} \\
  \And
  Aayushya Agarwal$^*$ \\
  Carnegie Mellon University\\
  Pittsburgh, PA 15213 \\
  \texttt{aayushya@andrew.cmu.edu} \\
  \And
  Yihan Ruan \\
  Carnegie Mellon University\\
  Pittsburgh, PA 15213 \\
  \texttt{yihanr@andrew.cmu.edu} \\
  \And
  Soummya Kar \\
  Carnegie Mellon University\\
  Pittsburgh, PA 15213 \\
  \texttt{soummyak@andrew.cmu.edu} \\
  \And
  Larry Pileggi \\
  Carnegie Mellon University\\
  Pittsburgh, PA 15213 \\
  \texttt{pileggi@andrew.cmu.edu} \\
  \And
  Bruno Sinopoli \\
  Washington University in St. Louis\\
  St. Louis, MO 63130 \\
  \texttt{bsinopoli@wustl.edu} \\
}
\else
\author{%
  David S.~Hippocampus\thanks{Use footnote for providing further information
    about author (webpage, alternative address)---\emph{not} for acknowledging
    funding agencies.} \\
  Department of Computer Science\\
  Cranberry-Lemon University\\
  Pittsburgh, PA 15213 \\
  \texttt{hippo@cs.cranberry-lemon.edu} \\
}
\fi

\begin{document}

\maketitle

\begin{abstract}
We present a stochastic first-order optimization method specialized for deep neural networks (DNNs), ECCO-DNN. This method models the optimization variable trajectory as a dynamical system and develops a discretization algorithm that adaptively selects step sizes based on the trajectory's shape. This provides two key insights: designing the dynamical system for fast continuous-time convergence and developing a time-stepping algorithm to adaptively select step sizes based on principles of numerical integration and neural network structure. The result is an optimizer with performance that is insensitive to hyperparameter variations and that achieves comparable performance to state-of-the-art optimizers including ADAM, SGD, RMSProp, and AdaGrad. We demonstrate this in training DNN models and datasets, including CIFAR-10 and CIFAR-100 using ECCO-DNN and find that ECCO-DNN's single hyperparameter can be changed by three orders of magnitude without affecting the trained models' accuracies. ECCO-DNN’s insensitivity reduces the data and computation needed for hyperparameter tuning, making it advantageous for rapid prototyping and for applications with new datasets. To validate the efficacy of our proposed optimizer, we train an LSTM architecture on a household power consumption dataset with ECCO-DNN and achieve an optimal mean-square-error without tuning hyperparameters.
\end{abstract}

\section{Introduction}

Stochastic gradient-based optimization is a central idea across many fields of research, science, and engineering. In particular, developing optimization methods for training deep neural networks (DNNs) is of interest, as DNNs have yielded great success across multiple domains including image classification \cite{9421942}, medicine \cite{bakator2018deep} and power systems \cite{9265470}. While good results from DNNs are attainable, the overall performance often relies on careful tuning of optimizer hyperparameters such as learning rate, weight decay, and momentum weight. Finding good hyperparameters in practice is often expensive both in terms of computation and data, as cross-validation routines require sufficient evaluation of the DNN across the hyperparameter values. While automated hyperparameter selection routines such as grid search, random search, or Bayesian optimization \cite{bergstra2011algorithms} can reduce manual effort, small variations in parameter values can result in large ranges in performance on the same DNN; for example, \cite{zhang2017intent} demonstrates a classification accuracy range of 33\% to 96\% when the hyperparameters are varied within an order of magnitude. In addition, optimal hyperparameter values do not always transfer between DNNs of different structures or to new datasets \cite{liao2022empirical}. This presents a bottleneck for rapid prototyping, training new models for unseen datasets, and applications with low computational resources.

In this work, we aim to reduce the computational effort for hyperparameter tuning by introducing ECCO-DNN: a closed-loop optimizer that achieves comparable performance to modern DNN optimizers without the need for extensive hyperparamter tuning. This is achieved by a two step process: first, the ordinary differential equations (ODEs) describing the optimization variables are related to a dynamical system, and designed for fast convergence to a steady state, which coincides with a critical point of the objective function. Second, the ODE solution is discretized using an explicit numerical integration method and step sizes are adaptively selected based on controlling properties of the numerical integration including local truncation error (LTE), and limiting in regions of fast gradient-change.  

The discretized ODE solution is controlled in a closed-loop manner to the known optimality condition to reduce the sensitivity to hyperparamter values. By controlling for numerical integration properties, we ensure the discretized waveform will follow the trajectory to the critical point value. 
The term "closed-loop" is used because the step sizes are adjusted based on the LTE they generated in the previous step as well as the rate of change of activation function gradients, thus yielding a feedback loop. In comparison to step size schedulers that diminish the learning rate based on the number of epochs, the numerical integration based control allows large step sizes to be taken when the gradient is flatter, and smaller step sizes when the gradient changes quickly. In addition, ECCO-DNN is conceptually different from other adaptive methods like AdaGrad and Adam, as we adapt the step size such that the discrete-time trajectory tracks the continuous-time trajectory accurately, but without sampling too many times. 



The result is an adaptive optimizer that has a single hyperparemter: a maximum LTE tolerance. We demonstrate in simulation that this parameter can be changed by several orders of magnitude with nominal effect on optimization performance. This method is validated by training multiple DNN models on the MNIST and CIFAR-10 datasets.  We find that ECCO-DNN achieves comparable classification performance to Adam, SGD, RMSProp, and AdaGrad. However, ECCO-DNN is \emph{highly insensitive} to its hyperparameter, whereas small perturbations to the parameters of the comparison methods yielded meaningless results. This demonstrates ECCO-DNN's advantage in achieving comparable performance without the need for any hyperparameter tuning. ECCO-DNN's insensitivity to hyperparameter variations is particularly beneficial for rapid prototyping, training new/unseen models or datasets, and applications with limited computational resources. We demonstrate these benefits in training an LSTM model on a dataset related to household power usage, for which there are no prior insights for comparison method hyperparameters. In comparison to state-of-the-art optimizers, ECCO-DNN yielded equivalent test performance while requiring significantly less training data and computation. ECCO-DNN is integrated within the PyTorch framework to provide generalized training for models across different datasets.

\section{Problem Formulation}
In this work, we consider the following unconstrained optimization problem:
    \begin{gather}
        \min_{\x} f(\x),\label{prob}\\
        \x^* \in \argmin_{\x} f(\x).
    \end{gather}
    where $\x\in \R^n$ is a parameter vector and $f: \R^n\to\R$ is an objective function dependent on the parameters such as a loss function. It is known \cite{brown1989some} that \eqref{prob} may be solved via the component-wise scaled gradient flow ODE initial value problem (IVP),
    \begin{equation}
        \dxdt(t)=-Z(\xt)^{-1} \nabla f(\xt),\quad \x(0) = \x_0. \label{scaled gradient flow}
    \end{equation}
    where $\dxdt(t)$ refers to the time derivative of $\xt$ and $Z^{-1}$ is a positive diagonal matrix. The solution to an ODE IVP is computed with the integral $\x(t) = \x(0) + \int_0^t \dxdt(s) ds.$ In this work we assume that the following set of assumptions hold:

    \begin{enumerate}[wide=\parindent,label=\textbf{(A\arabic*)}]
        \item $f\in C^2$ and $\inf_{\x\in\R^n}f(\x)>-R$ for some $R>0$. \label{a1}
        \item $f$ is coercive, i.e., $\lim_{\|\x\|\to\infty} f(\x) = +\infty$. \label{a2}
        \item (Bounded Hessian) For all $\theta\in\R^n$,  $\|\nabla^2 f(\theta)\|\preceq \lambda_{sup}$ for some $\lambda_{sup}>0$. \label{a3}
        \item (Lipschitz and bounded gradients) For all $\x,\y\in\R^n$,  $\Vert \dfdx-\dfdy\Vert\leq L\Vert\x-\y\Vert$, and $\Vert\dfdx\Vert\leq B$ for some $B>0$. \label{a4}
        \item $Z(\x)^{-1}$ is diagonal for all $\x$ and $d_2>Z_{ii}(\x)^{-1}>d_1$ for all $i,\ \x$ and for some $d_1,\ d_2>0$.\label{z def}
    \end{enumerate}

    
    Note that $f$ is not assumed to be convex. Coercivity guarantees that there exists a finite global minimum for $f(\x)$ that is differentiable \cite{peressini1988mathematics}. 
    
    \begin{definition}
    We say $\x$ is a \emph{critical point} of $f$ if it satisfies $\nabla_{\x}f(\x) = \vec{0}$. Let $S$ be the set of \emph{critical points}, i.e. $S=\{\x\ |\ \dfdx = \vec{0}\}$.
    \end{definition}
    
    The coercivity and differentiability of $f$ guarantee that any minima are within the set $S$.

\section{Method}
Our optimization method is called Equivalent Circuit Controlled Optimization for Deep Neural Networks (ECCO-DNN) as some of the continuous-time analysis and the numerical integration method draw upon circuit simulation ideas\footnote{Details on the circuit connection are provided in the Appendix.}. Section \ref{sec:gf} defines the gradient flow ODE, proposes a control policy, and establishes its convergence. The ODE is then solved for its steady-state, i.e. a critical point of the objective function in Section \ref{sec:disc} via a discretization based on Forward Euler integration. Our method controls the local truncation error and monitors quickly changing gradients produced by activation functions of the neural network. 

\subsection{Gradient Flow} \label{sec:gf}
We first establish convergence of the component-wise scaled gradient flow in \eqref{scaled gradient flow}. It can be shown that for any objective function, $f$, and any scaling function, $Z$, satisfying \ref{a1} - \ref{z def}, the gradient flow IVP will converge to some $\x \in S$. 

\begin{theorem} \label{convergence theorem}
Given \ref{a1}, \ref{a4}, \ref{z def}, consider the component-wise scaled gradient flow IVP in \eqref{scaled gradient flow}. Then,
\begin{equation}
    \lim_{t\to\infty}\Vert \dfdxt\Vert = 0.
\end{equation}
\end{theorem}

The proof is provided in Appendix \ref{main proof}. Given Theorem \ref{convergence theorem}, the next objective is to design a scaling matrix $Z$ satisfying \ref{z def} that yields faster convergence than the base case of $Z=I$. To tackle this objective, $Z$ is chosen to maximize the negative time derivative of the squared gradient:
\begin{align}
    &\max_{Z}- \frac{1}{2} \frac{d}{dt}\Vert \dfdxt \Vert^2 \\
    &= \max_Z \dfdxt^{\top}\nabla^2 f(\xt)Z^{-1}\dfdxt. \label{z criteria}
\end{align}
The purpose of this criteria is to design $Z$ to maximize the speed, i.e. rate of change, towards the optimality condition of $\nabla f(\x) = 0$, thus reaching a critical point fast. Define $G(\xt)$ be a diagonal matrix where the diagonal elements are the gradients: $G_{ii}(\xt) = \frac{\partial f(\xt)}{\partial \x_i(t)}$. To ensure a closed-form solution for $Z$, a regularization term is added to the objective:
\begin{equation}
 \max_{\z} \dfdxt^{\top}\nabla^2 f(\xt)G(\xt)\z-\frac{\delta}{2} \Vert \z\Vert^2,
\end{equation}
where $\delta$ is a scalar penalty (set to $\delta=1$, see Appendix \ref{main proof}) and $\z$ is a vector where $\z_i=Z^{-1}_{ii}$.
This regularized objective can then be solved for a functional form for $Z$:
\begin{equation}
    Z_{ii}(\xt)^{-1} = \max\{[G(\xt) \nabla^2 f(\xt)\dfdxt]_i,1\},\label{z true}
\end{equation}
where the maximization retains positivity and invertibility of $Z$. This derivation is in Appendix \ref{gf derivation}. 
\begin{lemma} \label{bounded z true}
    Assume \ref{a1}-\ref{a4} hold. The definition of $Z(\xt)^{-1}$ in \eqref{z true} satisfies \ref{z def}.\footnote{See proof in Appendix \ref{zbounded}.}
\end{lemma}
While \eqref{z true} is an optimal solution for $Z$ based on \eqref{z criteria}, it is a second order method, and is therefore not suited to neural network training. To avoid computation of the Hessian, we instead derive a first order $Z$ based on the following approximation of the time derivative of the gradient at $\Delta t$:
\begin{align}
    \ahat(\xt)\triangleq \frac{\dfdxt - \nabla f(\x (t-\Delta t))}{\Delta t}\approx \frac{d}{dt}\dfdxt. \label{ahat}
\end{align}
The limit as $\Delta t\to 0$ is exactly equal to desired quantity $\frac{d}{dt}\nabla f(\xt)$, making this an apt approximation for small $\Delta t$. With this substitution, the following approximation to \eqref{z true} can be derived, as shown in Appendix \ref{z approx deriv}. 
\begin{equation}
     \widehat{Z}_{ii}(\xt)^{-1}=\sqrt{\max\{[-G(\xt)\ahat(\xt)]_i,1\}}. \label{z approx}
\end{equation}

\begin{lemma}\label{bounded z approx}
    Assume \ref{a1}-\ref{a4} hold. The definition of $Z(\xt)^{-1}$ in \eqref{z approx} satisfies \ref{z def}.\footnote{See proof in Appendix \ref{zhat bounded}.}
\end{lemma}
\begin{lemma}\label{z approx complexity}
    Given the gradients, the computation complexity to evaluate $\widehat{Z}^{-1}$ at $\xt$ is $\mathcal{O}(n)$.\footnote{See proof in Appendix \ref{complexity approx}.}
\end{lemma}

\subsection{Discretization}\label{sec:disc}
Given a gradient flow known to converge to a critical point, the next step is to solve the ODE for this value. To accomplish this goal, we discretize the continuous-time trajectory and evaluate the state of the ODE, $\xt$ at each time-point as,
\begin{equation}
\x(t+\Delta t) = \x(t) + \int_t^{t+\Delta t} \dot{\x}(s)ds.\label{eq:ode_int}
\end{equation}
This integral, however, generally does not yield an analytical solution and is instead approximated via numerical integration methods. In this work, we apply an explicit Forward-Euler (FE) method, which approximates $\x(t+\Delta t)$ by $\x(t+\Delta t) = \x(t) + \Delta t \dot{\x}(t).$ For our scaled gradient flow definition in \eqref{scaled gradient flow}, the FE discretization is,
\begin{equation}
\x(t+\Delta t) = \x(t) - \Delta t Z(\x(t))^{-1} \nabla f(\x(t)). \label{eq:gd_flow_fe}
\end{equation}
Beginning at $t=0$, let $k$ be the number of discrete time steps and $\alpha_k \triangleq (\Delta t)_k$. Then, the trajectory expressed in discrete time is $\x_{k+1} = \x_k - \alpha_k Z(\x_k)^{-1} \nabla f(\x_k)$.
While FE is easy to evaluate (note that defining $Z\equiv \mathbf{I}$ and applying FE yields general gradient descent with step size $\Delta t$), it is prone to numerical issues of low accuracy and skipping regions of interest \cite{pillage1995electronic}. In this work, we draw on ideas from dynamical systems analysis to combat these problems. In the following sections the step size will be developed as,
\begin{equation}
    (\Delta t)_k = \delta_k (\Delta t_{\tau})_k,
\end{equation}
where $\Delta t_{\tau}$ is selected to preserve accuracy based on measuring the approximation error, and $\delta\in(0,1]$ scales down $\Delta t_{\tau}$ to ensure regions with quickly changing gradients are not missed.

\subsubsection{Adaptive Step Size for Accuracy}
The FE approximation is known to have low accuracy.
For the gradient flow, the error in the approximation is a problem that causes the discretization to diverge from the true trajectory and the critical point cannot be found.
To ensure that the continuous-time trajectory is closely followed, we apply techniques in circuit simulation \cite{pillage1995electronic} that estimate the local truncation error (LTE) of potential step sizes, and only adopt a $\Delta t$ if the resulting error is bounded by a given tolerance $\eta$ \cite{pillage1995electronic}. The LTE produced a step $\Delta t$ under FE discretization can be approximated by \cite{pillage1995electronic} as, 
\begin{align*}
	\tau = 0.5 \Delta t | &Z(\x(t-\Delta t))^{-1} \nabla f(\x(t-\Delta t)) - \\  &Z(\x(t))^{-1}\nabla f(\x(t))|,
\end{align*} 
where $\tau_i$ is the error for dimension $i$. From the optimization perspective, the LTE of FE is a measure of the change in gradient over a single iteration. Neglecting the LTE can result in divergence from the trajectory and lead to either gradient overflow or oscillation around a critical point. For instance, consider optimizing the Rosenbrock function \cite{testfunc} using a  gradient flow and a standard FE ODE solver with a fixed time step of $2e-3$, which violates a LTE check. As shown in Figure \ref{fig:matlab_response_rosenbrock}, this trajectory fails to converge to a critical point. However, by selecting a $\Delta t$ that satisfies a LTE condition, we can ensure that the discrete trajectory follows the continuous-time trajectory to the critical point.

\ifarxiv
\begin{figure}[h]
    \centering
\includegraphics[width=.4\linewidth]{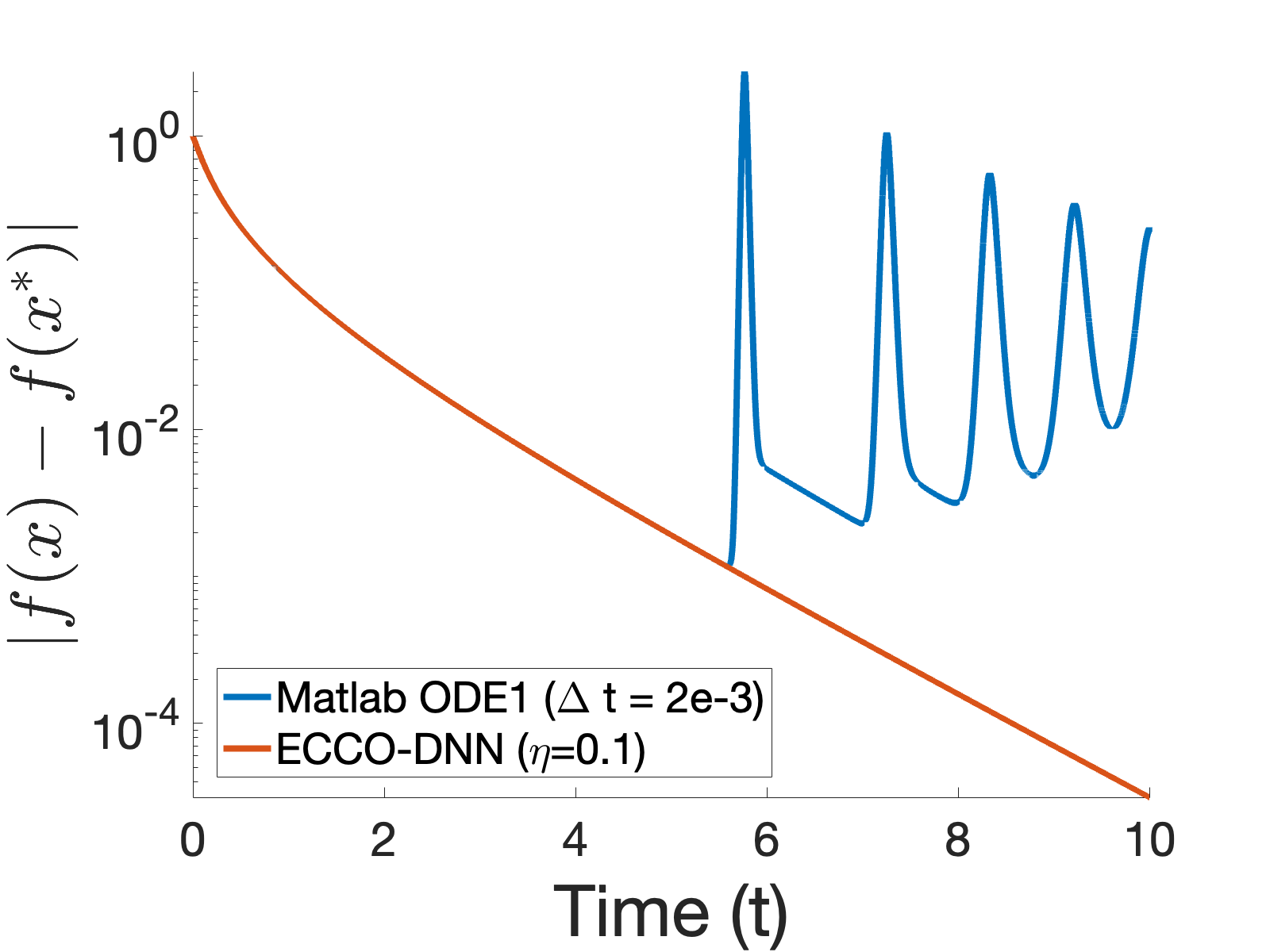}
  \captionof{figure}{Simulating the transient response of Rosenbrock \cite{andrei2008unconstrained} using Matlab ODE1 (F.E. numerical integration) with a fixed time step of $\Delta t=2e-3$, which violates the LTE condition of $\eta=0.1$ and ECCO-DNN which adaptively conforms to the LTE condition}
\label{fig:matlab_response_rosenbrock}
\end{figure}
\begin{figure}
    \centering
\includegraphics[width=.4\linewidth]{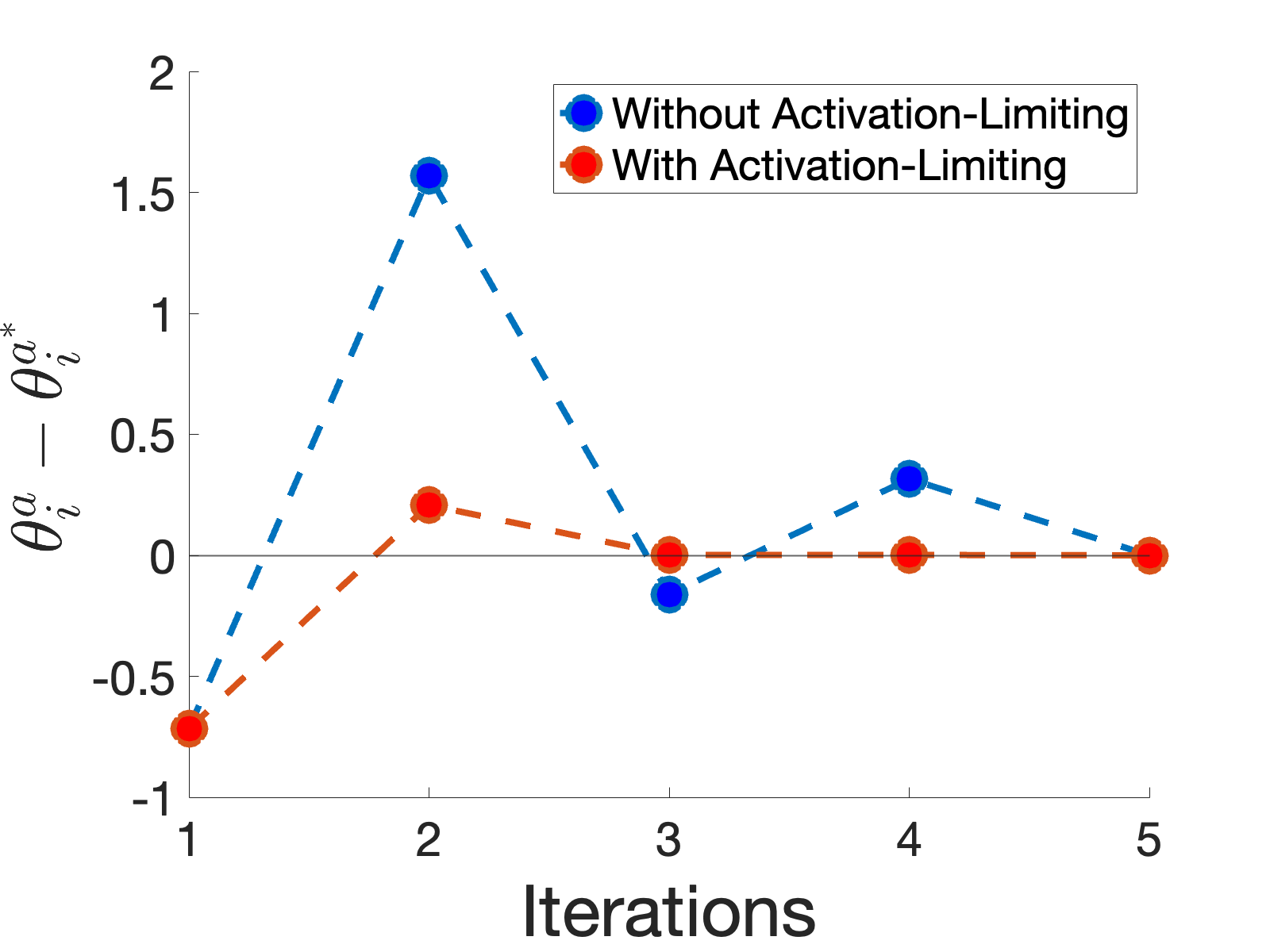}
  \captionof{figure}{ During the training of MNIST data, the input state $\xacti$ to a ReLU neuron in the DNN model \cite{mnist_model} exhibits oscillations around the optimal point $\theta_i^{a^*}$ when the activation function limiting is not applied and takes longer to converge.}
  \label{fig:limiting}
\end{figure}
\else
\begin{figure}[h]
    \centering
\includegraphics[width=.6\linewidth]{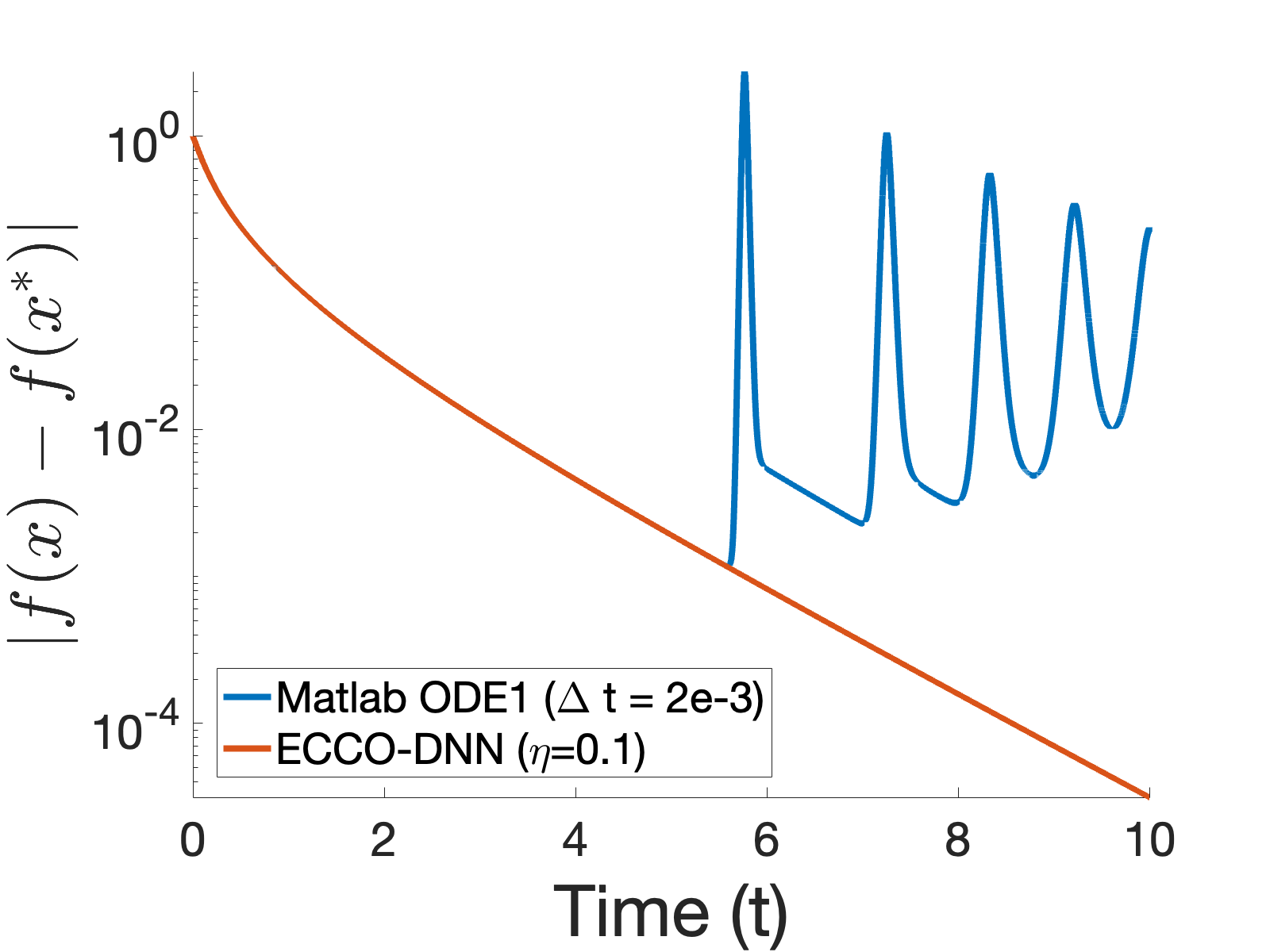}
  \captionof{figure}{Simulating the transient response of Rosenbrock \cite{andrei2008unconstrained} using Matlab ODE1 (F.E. numerical integration) with a fixed time step of $\Delta t=2e-3$, which violates the LTE condition of $\eta=0.1$ and ECCO-DNN which adaptively conforms to the LTE condition}
\label{fig:matlab_response_rosenbrock}
\end{figure}
\begin{figure}
    \centering
\includegraphics[width=.56\linewidth]{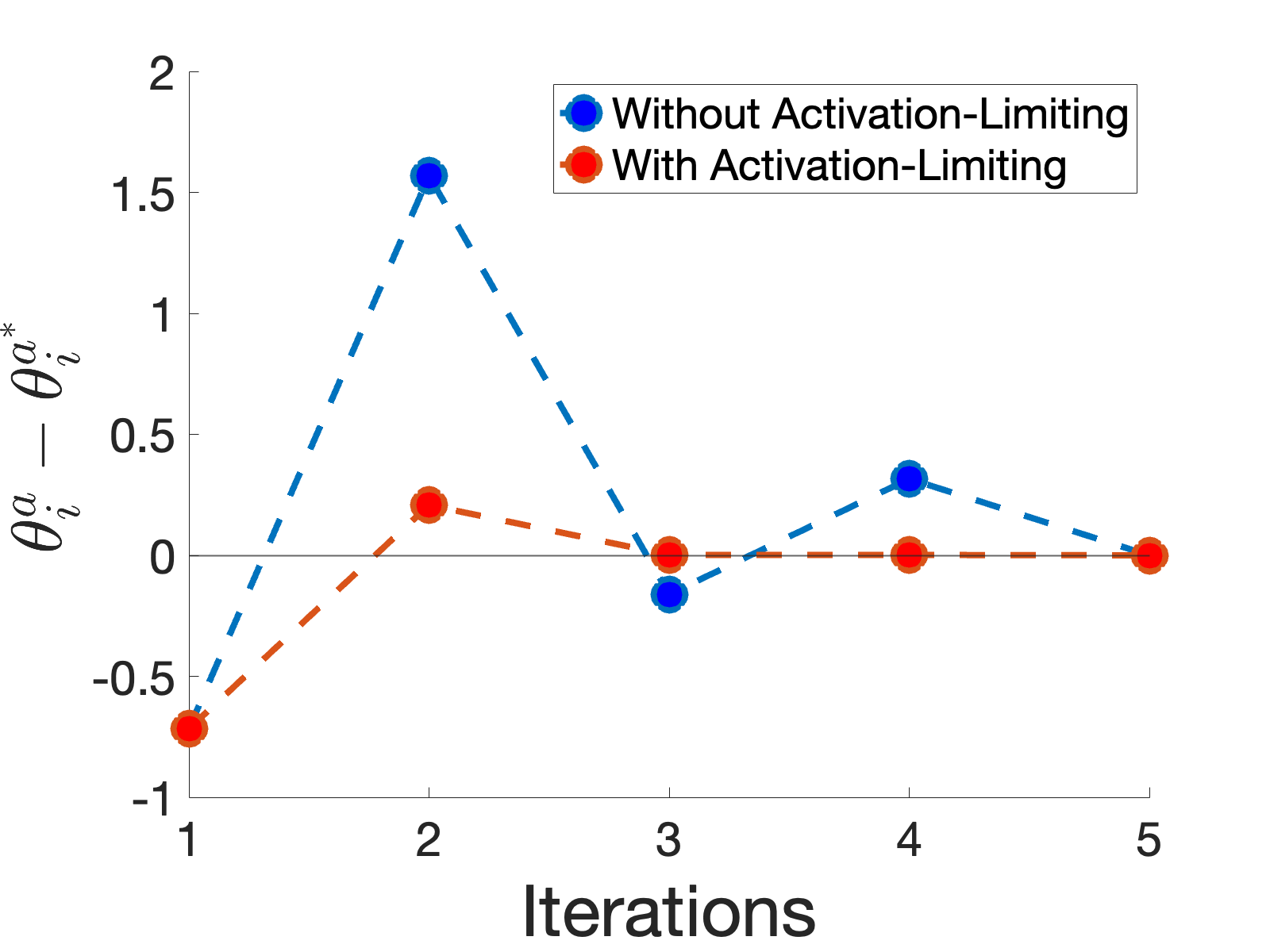}
  \captionof{figure}{ During the training of MNIST data, the input state $\xacti$ to a ReLU neuron in the DNN model \cite{mnist_model} exhibits oscillations around the optimal point $\theta_i^{a^*}$ when the activation function limiting is not applied and takes longer to converge.}
  \label{fig:limiting}
\end{figure}
\fi


To reach the critical point quickly, we want to find large step sizes that produce LTE satisfying the given tolerance $\eta$. To accomplish this, we first calculate the error $\tau$ using the time-step from the previous iteration, $(\Delta t)_{k-1}$. We then scale $(\Delta t)_{k-1}$ by $\frac{\eta}{\max(\tau)}$ to propose a step size based on LTE:
\begin{align}
    \Delta t_{\tau}&=\min\left\{\eta\|\tau\|^{-1}_{\infty}(\Delta t)_{k-1},1\right\},\\
    &= \min\left\{2\eta \|(\Delta \x)_{k-1} - (\Delta \x)_k\|_{\infty}^{-1},1\right\},
\end{align}
where $(\Delta \x)_k = Z(\x_k)^{-1}\nabla f(\x_k)$. The factor, $\frac{\eta}{\max(\tau)}$, maximally scales the time-step to ensure the maximum local truncation error is close to the tolerance, $\eta$. This routine diminishes the step size when the LTE condition is violated, and scales up the step size when the maximum LTE is below the tolerance. Bounding $\Delta t_{\tau} \leq 1$ also ensures we do not encounter gradient overflow and follows the bounds for learning rate in other optimizers \cite{kingma2014adam}.



\subsubsection{Adaptive Step Size for Quickly Changing Gradients}
The second issue is that FE discretization may fail to track the continuous-time trajectory in regions where the gradient changes quickly, e.g. with large local Lipschitz constants, resulting in divergence or numerical oscillation. Given Theorem \ref{convergence theorem}, it is desirable to select the largest $\Delta t$ possible at each iteration to reach the critical point quickly. If a proposed $\Delta t$ yields gradients such that $\nabla f(\xt)\approx \nabla f(\x(t+\Delta t))$, then the LTE will appear to be small. However, if there exists some $t'\in (t, t+\Delta t)$ such that $\nabla f(\x(t'))$ is very different to $\nabla f(\xt)$ and $\nabla f(\x + \Delta t)$, then the proposed $\Delta t$ is too large, and the discretized trajectory will diverge from the continuous trajectory. 

Issues related to quickly changing gradients are generally avoided with schedulers that diminish the learning rate to achieve sufficiently small step sizes. 
However, such schedulers can be inefficient because open-loop implementations constantly diminish the step size rather than adapting the step sizes based on the topology of the gradient.

%



Luckily, neural networks have an explicit structure, allowing an adaptive step-size selection routine to be developed that does not skip over regions with quickly changing gradients. In DNNs, this issue arises in the activation layers, which exhibit the highest rates of change for the gradient. Consider the gradients of the sigmoid, ReLu, and tanh activation functions illustrated in Figure \ref{fig:activation_gradients}, where large FE steps initialized at the green (left) dot may jump to the red (right) dot. Although they may satisfy LTE checks, such time steps clearly ignore the central regions of the activation functions and may cause premature saturation.


\begin{figure*}
    \centering
    \includegraphics[width=0.7\textwidth]{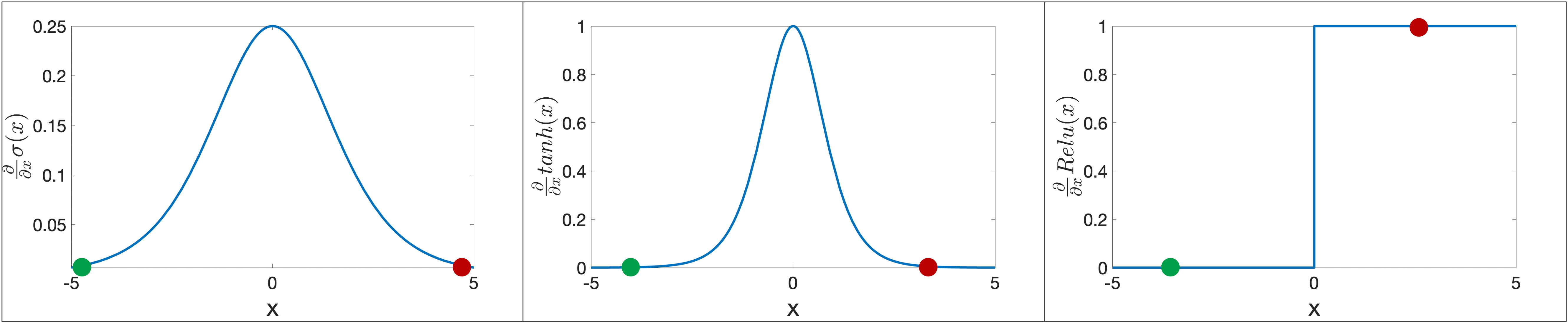}
    \caption{Gradients of three activation functions (sigmoid, ReLu and tanh) exhibit fast changes at $\xact=0$. Large step-sizes can cause a step from the green dot to the red dot in a single iteration, thereby skipping the entire region of interest around $\xact=0$.}
    \label{fig:activation_gradients}
\end{figure*}

To tackle the issue of quickly changing gradients, we take inspiration from circuit simulation. This field has achieved large-scale simulation by developing adaptive step sizes to properly traverse the gradient space of nonlinear diode and transistor models \cite{nagel1975spice2}. In a similar fashion, ECCO-DNN selects a time step to ensure the activation layer does not skip the region where the maximum rate of change of the gradient occurs. We denote the parameters that are being optimized for a DNN as $\x=\begin{bmatrix}
    W \\ B
\end{bmatrix}$, which consists of the weights, $W$, and biases, $B$, for each layer. The vector of inputs to the activation layer is denoted as $\xact=W\hat{\theta} + B$, where $\hat{\theta}$ is the output of the previous layer. We observe the maximum rate of change of the gradient occurs at $\xact=0$, as shown in Figure \ref{fig:activation_gradients}. Therefore, ECCO-DNN modifies the time steps to limit the input of the activation layer from skipping over the point $\xact=0$ for the activation functions.  Note in Figure \ref{fig:limiting}, the trajectory without limiting oscillates around the optimal point and takes five iterations to converge, whereas the limited step takes three iterations.

This limiting step will only be applied to the parameters of layers with commonly used activations (sigmoid, ReLu, and tanh) because they have large gradients around $\xact=0$ and are most prone to skipping the middle region. Each element of the vector, $\xact$, is denoted as $\xacti$. If for a proposed $\Delta t_{\tau}$ it holds that $\xacti(t)<0$ and $\xacti(t+\Delta t_{\tau})>0$ (or vice versa), then we want to \emph{limit} the step by selecting $\delta\in(0,1)$ such that $\xacti(t+\delta \Delta t_{\tau}) = 0$. This condition checks if the proposed time step results in a new $\xacti$ that skips over the middle region of the activation function. By forcing $\xacti(t+\delta \Delta t_{\tau}) = 0$, we ensure that the middle region is visited. If this condition holds for multiple $i$, the $\delta$ is chosen such that $\xacti(t)$ and $\xacti(t+\delta \Delta t_\tau)$ do not have opposite signs for all $i$. 
To enforce that $\xacti(t+\Delta t)=0$, first define $\Delta \xact=\Delta W \hat{\theta} + \Delta B$, where $\Delta \x=\begin{bmatrix}
    \Delta W \\ \Delta B
\end{bmatrix} = -\Delta t_{\tau}Z_{ii}(\x(t))^{-1}\frac{\partial}{\partial \x}f(\x).$
Then, the $\delta_i$ that forces $\xacti(t+\delta_i \Delta t_{\tau})=0$, is $\delta_i = -\xacti(t)/\Delta \xacti$. As the $\delta$ for the entire layer must be selected such that $\xacti(t)$ and $\xacti(t+\delta \Delta t_{\tau})$ do not have opposite signs  for all $i$, it follows that $\delta = \min\{1, \delta_i\}$. This step size limiting procedure is used to update the parameters, $\x$ (i.e., the weights and biases) of layers with sigmoid, ReLu, or tanh activation functions. Other layers are assigned $\delta_i\equiv 1$.

\section{Algorithm}
In this section, we summarize how to use scaled gradient flow and F.E. integration to solve optimization problems in Algorithm \ref{bigger algo}. The proposed algorithm is a closed-loop optimizer that adaptively selects $\alpha_k$ with a single hyperparameter, the LTE tolerance $\eta$. A sample of the stochastic gradient is observed on line \ref{alg:sample}. This is used to calculate the scaling matrix on line \ref{alg:control}. As the limiting condition for activation functions requires an explicit definition of the structure from the user, the limiting conditions in line \ref{alg:delta} is optional. If the user provides the structure of the DNN model and the parameters of interest are of the layer preceding the activation function, which may violate the limiting condition, then line \ref{alg:delta} calculates the $\delta_i$ values, which are otherwise set to 1. The overall step size is updated on line \ref{alg:step}, and the parameters are updated on line \ref{alg:update}. We find that $\eta=0.1$ and normalizing the diagonal elements of $\hat{Z}^{-1}$ works well in practice.



\begin{algorithm}[h]\small
 \SetAlgoLined
 \LinesNumbered
 \SetKwInOut{Input}{Input}
 \Input{Stochastic objective function with parameters $\x$: $f(\x)$,
 Gradient of objective function: $\nabla_{\x}f(\x)$,\\
 Initialization: $\x_0$,
 LTE Tolerance: $\eta> 0$,
 Convergence Condition: $\epsilon> 0$}
 \SetKwProg{Function}{function}{}{end}
 
 $k\gets 0$\;
 
 \SetKwRepeat{Do}{do}{while}
    
     \Do{$\|f(\x_k)-f(\x_{k-1})\|> \epsilon$}
     { 
     $k\gets k+1$\;

     $g_k\gets \nabla f(\x_k)$\;\label{alg:sample}

     $G_k\gets \text{diag}(g_k)$\;
     
     $\hat{Z}^{-1}_k \gets \text{diag}\left(\sqrt{\max\{\mathbf{1}, \alpha_{k-1}G_k\big(g_{k-1} - g_k\big)\}}\right)$\;\label{alg:control}




     $(\Delta \x)_k \gets \hat{Z}^{-1}_k g_k$\;

     $\Delta t_{\tau}\gets \min\left\{2\eta \|(\Delta \x)_{k-1} - (\Delta \x)_k\|_{\infty}^{-1},1\right\}$\;
     

     
     \eIf{Layer with sigmoid, ReLu, or tanh activation function}
     {$\delta_i \gets -\xacti(t)/\Delta \xacti, \ \forall i$\label{alg:delta}}
     {$\delta_i \gets 1$}
     
     $\alpha_k\gets  \min\left\{1,  \delta_i\right\} \Delta t_{\tau}$\;\label{alg:step}
    
    $\x_{k+1}=\x_k-\alpha_k (\Delta \x)_k$\;\label{alg:update}
     }
     \Return{$\x_k$}
 \caption{Equivalent Circuit Controlled Optimization for Deep Neural Networks (ECCO-DNN)}\label{bigger algo}
\end{algorithm}


\section{Related Work}
Numerous adaptive stochastic optimization methods have been proposed to provide an adaptive optimizer for training DNNs including Adam \cite{kingma2014adam}, AdaGrad \cite{duchi2011adaptive}, AdaDelta \cite{zeiler2012adadelta} and RMSProp \cite{rmsproptieleman2012lecture}. Further improvements to these methods include AdamW \cite{adamw}, AdamP \cite{heo2020adamp} and AdaBelief \cite{zhuang2020adabelief}. Additionally, adaptive learning rate methods for second-order optimization methods including AdaHessian \cite{yao2021adahessian}, BFGS \cite{bfgs} and SR1 \cite{wright2006numerical}.

As AdaGrad defines a family of well-known optimization methods that adapt the learning rate based on previous gradient estimates, we next directly compare the ECCO-DNN update to that of AdaGrad. Adagrad defines the rule $\x_{k+1,i} = \x_{k,i} - \Big(\beta/\sqrt{\epsilon+\tilde{G}_{k,ii}}\Big)g_{k,i}$ where $g_k$ is the gradient, $\tilde{G}_k = \sum_{k'=1}^k g_{k'}g_{k'}^{\top}$ is the sum of outer products, $\beta$ is the learning rate, and $\epsilon$ is a small positive quantity. In comparison, the overall ECCO-DNN rule is $\x_{k+1,i} = \x_{k,i} -  \min\left\{1, \max_i \delta_i\right\}\min\left\{2\eta \|(\Delta \x)_{k-1} - (\Delta \x)_k\|_{\infty}^{-1},1\right\}\allowbreak \sqrt{\max\{1, \alpha_{k-1}g_{k,i}\big(g_{k-1,i} - g_{k,i}\big)\}} g_{k,i}$, where the first min term arises due to the activation function limiting, the second due to LTE control, and the square root term due to the trajectory control. While these are both adaptive methods, AdaGrad adapts the learning rate based on frequency of the features in the denominator term, and ECCO-DNN adapts both the trajectory (based on the continuous-time ODE control) and the time step (based on numerical integration accuracy and stability).

Gradient flow methods have been well-studied due to their potential to draw general conclusions in the continuous-time domain \cite{behrman1998efficient}, \cite{attouch1996dynamical}, \cite{brown1989some}. Gradient flow can be viewed as a dynamical system, thereby introducing concepts from control theory \cite{helmke2012optimization}, \cite{yuille1994statistical}, including Lyapunov theory which has emerged as a tool to show convergence of gradient flow \cite{cortes2006finite}, \cite{wilson2018lyapunov}, \cite{wilson2021lyapunov}, \cite{polyak2017lyapunov}, \cite{hustig2019robust}.

While useful theoretical results can be established in continuous-time, solving the ODE with a computer generally necessitates some discretization scheme. Recent advances have used sophisticated explicit integration techniques to approximate the continuous system \cite{pmlr-v97-muehlebach19a}, \cite{lin2016distributed}, \cite{andrei2004gradient}, \cite{scieur2017integration},\cite{maleki2021heunnet}, as well as multi-step Runge-Kutta integration methods \cite{ayadi2021stochastic},\cite{stillfjord2023srkcd},\cite{eftekhari2021explicit}. Further discretization methods have been explored including implicit integration methods such as backward-Euler \cite{barrett2020implicit},\cite{scieur2017integration}. Few of these works \cite{barrett2020implicit},\cite{scieur2017integration}, \cite{andrei2004gradient} have used notions of integration error to analyze existing first-order methods or to propose bounds on step sizes. 

Two methods, in particular, have developed adaptive step size routines using the dynamical system model. The first method applies a numerical integration method known as Heun’s method to adaptively select step sizes via a predictor-corrector rather than considering local truncation error \cite{wadiaoptimization,maleki2021heunnet}. However, Heun’s method is an explicit method that, without control of the LTE, is prone to divergence from the continuous time trajectory. 
In addition, the method proposed in \cite{li2017stochastic} learns an optimal policy for an annealing scheduler using a feedback policy rather than defining the approximation error in their Euler-integration step. However, this relies on decreasing the time-step during training, leading to open-loop control of the optimization process.

ECCO-DNN sets itself apart from adaptive dynamical system methods by employing FE integration and dynamically scaling step sizes based on the LTE and the activation functions used in the network structure. Unlike the other two methods, our approach is guided by the accuracy errors of FE, a practice commonly used in transient simulation tools like circuit simulation \cite{ltspice} and adaptive time-stepping algorithms in the field of numerical methods \cite{fehlberg1970classical,dormand1980family,kimura2009dormand,gragg1964generalized,chandio2010mproving}. Our approach differs from existing adaptive time-stepping algorithms \cite{fehlberg1970classical,dormand1980family,kimura2009dormand,gragg1964generalized,chandio2010mproving} by also limiting the step size based on the neural network structure. Many of the insights are derived from an equivalent circuit model of the dynamical system (as shown in Section \ref{sec:equivalent-circuit}), which has been proposed in fields like circuit simulation \cite{nagel1975spice2} and power systems \cite{jereminov2019equivalent}. We differ from the equivalent circuit model in \cite{jereminov2019equivalent} by designing an equivalent circuit of a dynamical gradient-flow for general optimization, rather than a circuit representation of a power grid.

\section{Simulations}
\label{sec:results}
In this section, we use ECCO-DNN to train multiple DNN models on MNIST, CIFAR-10 and a newly published dataset characterizing an individual household power consumption. We demonstrate that ECCO-DNN achieves comparable classification accuracy to the state-of-the-art optimizers Adam, SGD, RMSprop and AdaGrad for classifying MNIST and CIFAR-10, henceforth referred to as the comparison methods. Adam and SGD are implemented with a cosine annealing scheduler \cite{scheduler_cosine} whose hyperparameters are optimally tuned via grid search. The full list of hyperparameters and operational ranges are provided in Tables \ref{tab:adam_hyperparameters}-\ref{tab:ec_hyperparameters} in the Appendix. Furthermore, we show that ECCO-DNN's  single  hyperparameter ($\eta$) can be changed by several orders of magnitude without affecting the classification accuracy, whereas the similar perturbations to the hyperparameters of the comparison methods result in poor performance. Finally, ECCO-DNN is used to train a LSTM model on a household power consumption dataset. This experiment demonstrates that ECCO-DNN requires significantly less effort to tune for scenarios in which there is no pre-existing knowledge about good operating regions for hyperparameters of comparison methods.



\subsection{Training 4-Layer DNN to Classify MNIST Data}
\label{sec:mnist_experiment}
In this experiment, we train a 4-layer neural network provided by \cite{mnist_model} to classify MNIST data using a minibatch size of 1000. A random search was conducted to find good hyperparameters for each optimizer and scheduler, and the results are listed in Tables \ref{tab:adam_optimal_hyperparameters}-\ref{tab:ec_optimal_hyperparameters} in the Appendix. 
The training loss over 20 epochs using the optimally tuned hyperparameters is plotted in Figure \ref{fig:mnist_training}. Notably, in this experiment, $\eta$ for ECCO-DNN was not optimized and utilized a fixed LTE tolerance value of 0.1. We observe that ECCO-DNN achieves similar performance to the tuned optimizers  and achieved a comparable training classification accuracy of ~99\%.

To test the insensitivity to hyperparameters, we perturbed the hyperparameters of each optimizer within a normalized ball of radius of $\varepsilon = 0.1$ around the optimal hyperparameter values. The value of $\epsilon$ dictates the range of the hyperparameter selection as a percentage of the total operating range of a hyperparameter and is useful to characterize the perturbation for ECCO-DNN’s tolerance hyperparameter which has a much larger operating range.
 Two hundred hyperparameter values are sampled from the uniform distribution $\Tilde{\theta}\sim U(max\{\theta^* - \varepsilon(\overline{\theta} - \underline{\theta}),\underline{\theta}\}, min\{\theta^* + \varepsilon(\overline{\theta} - \underline{\theta}),\overline{\theta}\})$ which is bounded within the operational range of the hyperparameter values, $[\underline{\theta},\overline{\theta}]$. Each selection of hyperparameters was trained for 20 epochs and the resulting classification accuracy is shown in Figure \ref{fig:mnist}. ECCO-DNN reliably trained the model for any perturbation to its hyperparameter and achieves near-optimal training accuracies. However, the comparison methods were highly sensitive to hyperparameter selections as the classification accuracies ranged from 9.8\% to 99.2\%. These results demonstrate that the comparison methods require the performed grid search to have a fine granularity, thus requiring more computation and data for cross-validation. In comparison, ECCO-DNN is insensitive to its hyperparameter value and can be applied directly.

\ifarxiv
\begin{figure}
  \centering
  \includegraphics[width=.4\linewidth]{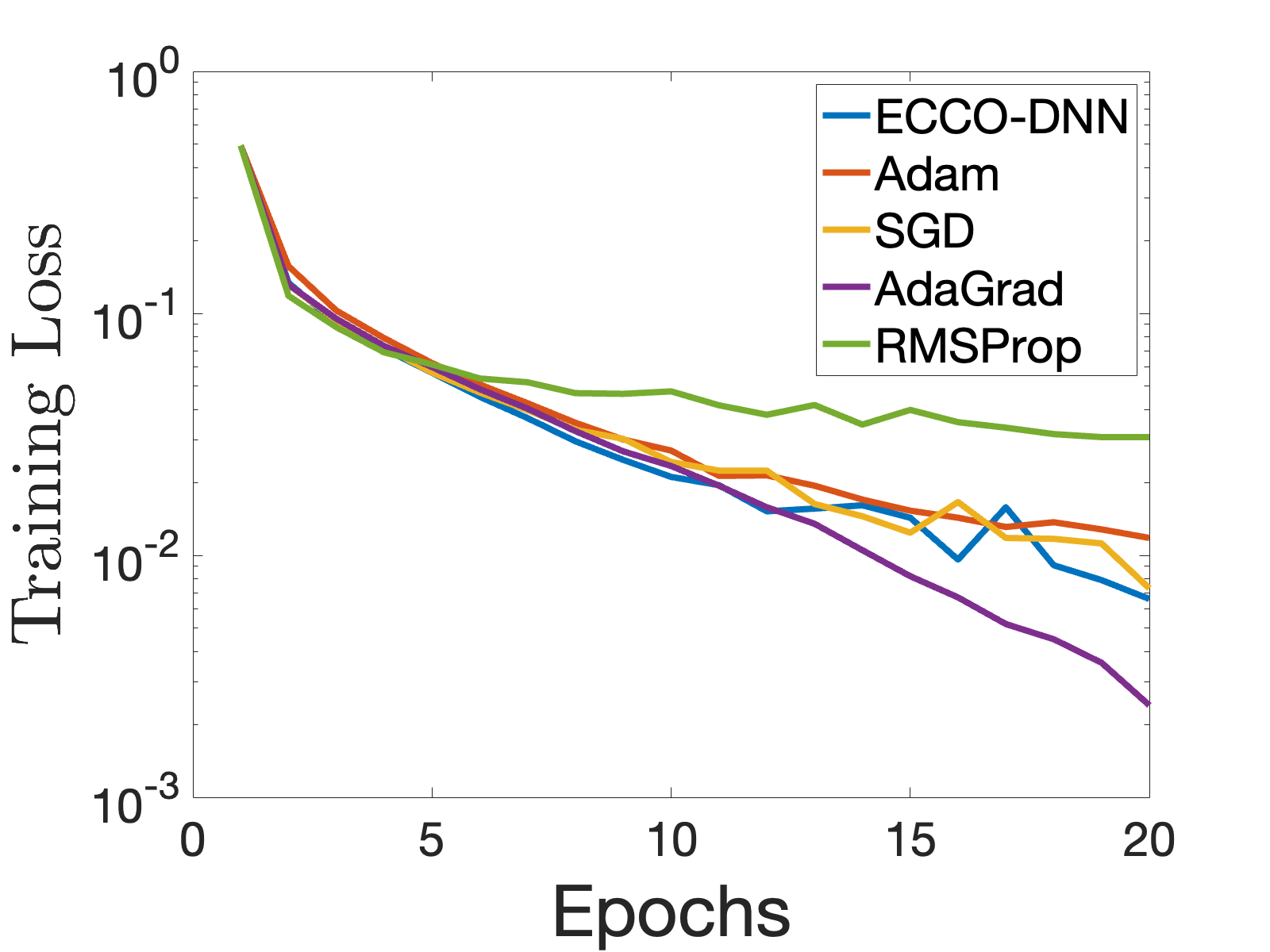}
  \captionof{figure}{Training Loss for DNN Classifying MNIST Data using Optimizers over 20 Epochs}
  \label{fig:mnist_training}
\end{figure}
\begin{figure}
  \centering
  \includegraphics[width=.5\linewidth]{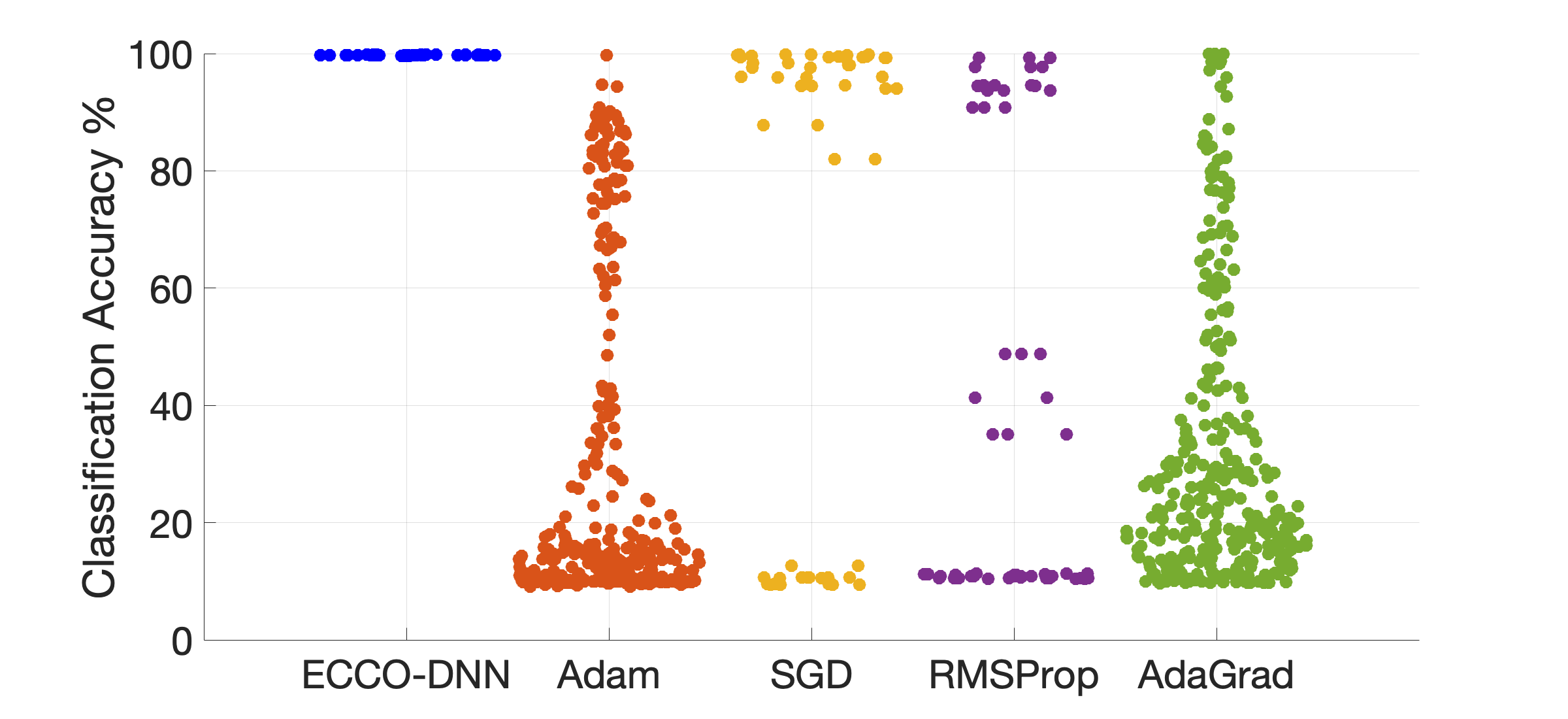}
  \captionof{figure}{Random search for Hyperparameters for classifying MNIST data using 4-layer DNN}
  \label{fig:mnist}
\end{figure}
\else
\begin{figure}
  \centering
  \includegraphics[width=.65\linewidth]{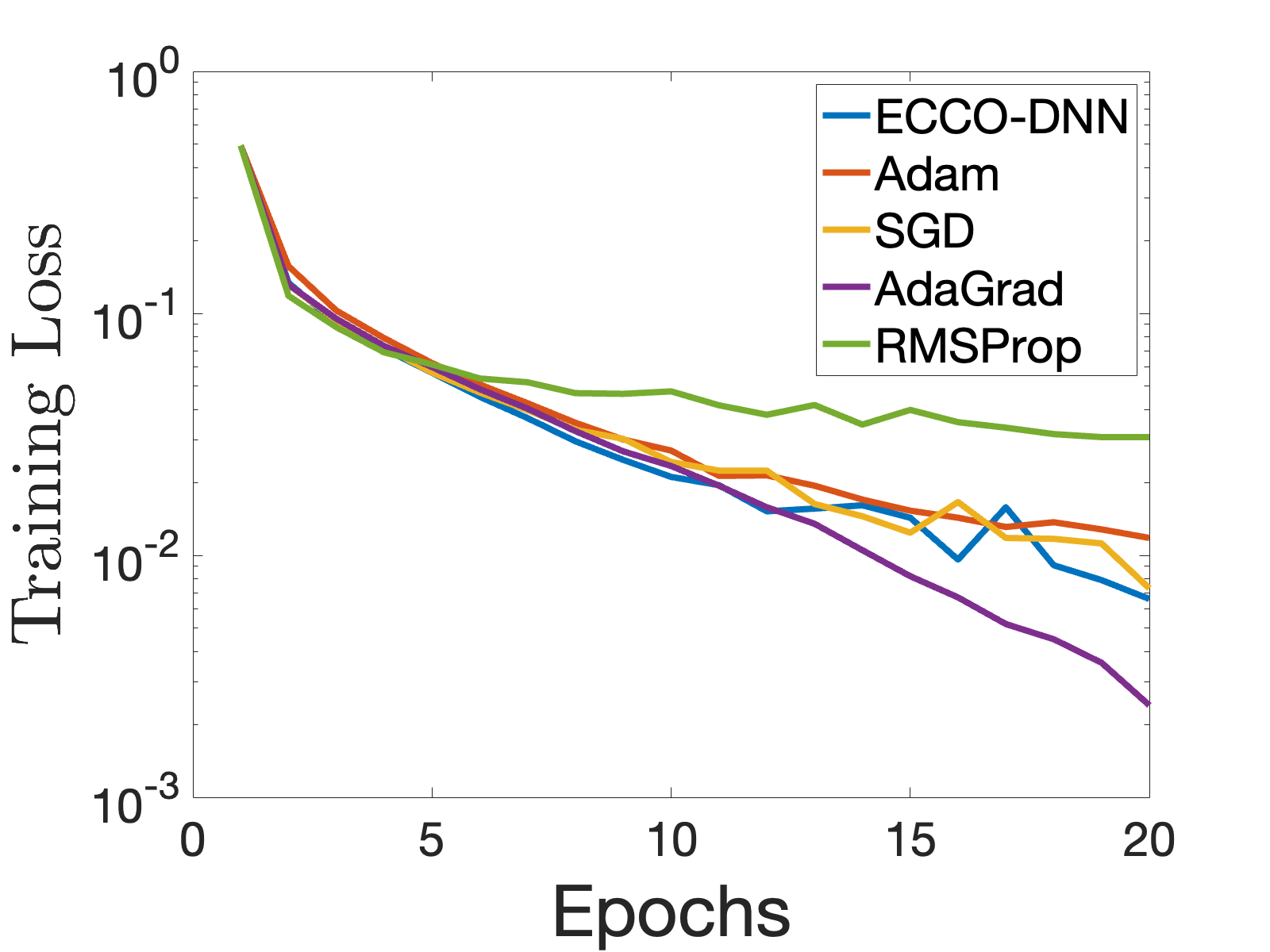}
  \captionof{figure}{Training Loss for DNN Classifying MNIST Data using Optimizers over 20 Epochs}
  \label{fig:mnist_training}
\end{figure}
\begin{figure}
  \centering
  \includegraphics[width=.8\linewidth]{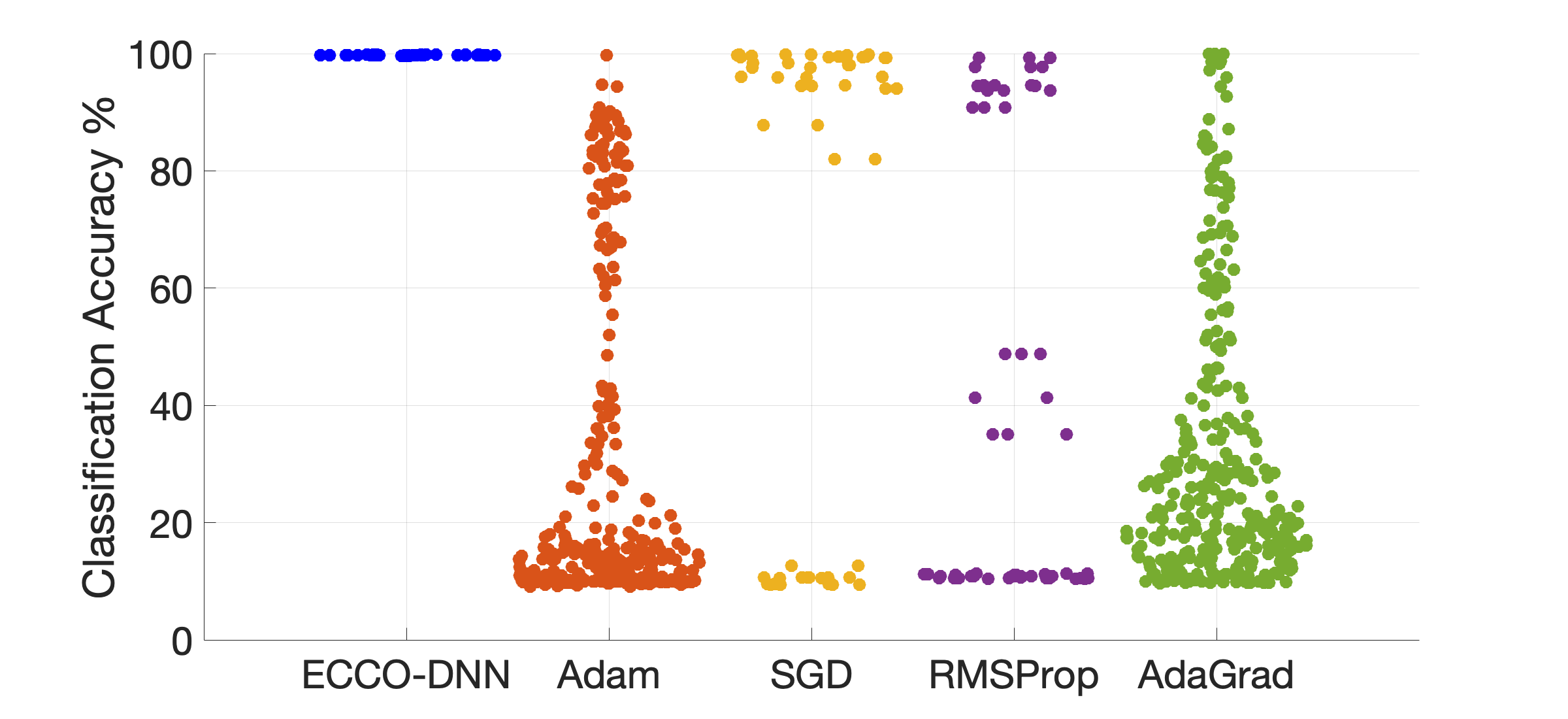}
  \captionof{figure}{Random search for Hyperparameters for classifying MNIST data using 4-layer DNN}
  \label{fig:mnist}
\end{figure}
\fi

\subsection{Training DNN models to Classify CIFAR-10}
\label{sec:cifar_experiment}
The sensitivity to hyperparameter selections is next demonstrated by training DNN models to classify CIFAR-10. We trained three models (Resnet18 \cite{resnet}, Lenet5 \cite{lenet}, and SeNet-18 \cite{senet}) using the ECCO-DNN and the comparison methods with a minibatch size of 128 for 20 epochs. The models used different activation functions in the hidden layers, which demonstrates the applicability of ECCO-DNN and the activation-specific limiting on different models.\footnote{The optimally tuned hyperparameter values are reported in Tables \ref{tab:adam_optimal_hyperparameters}-\ref{tab:ec_optimal_hyperparameters} in the Appendix.}

\ifarxiv
\begin{table}
  \caption{Final Test Accuracy of CIFAR-10 Experiment Trained by ECCO-DNN Versus  Tuned Comparison Methods}
\label{tab:cifar10_acc}
  \centering
  \resizebox{0.5\columnwidth}{!}{%
  \begin{tabular}{llllll}
    \toprule
      & ECCO-DNN & Adam & SGD & RMSProp & AdaGrad \\
    \midrule
     Resnet18 & 92.7 & 95.9 & 94.3 & 92.8 & 92.9 \\ \midrule
     Lenet5 & 61.1 & 60.4 & 63.1 & 54.5 & 62.5 \\ \midrule
     SeNet-18 & 94.1 & 89.4 & 95.3 & 93.3 & 92.7 \\ 
    \bottomrule
  \end{tabular}%
  }
\end{table}
\else
\begin{table}
  \caption{Final Test Accuracy of CIFAR-10 Experiment Trained by ECCO-DNN Versus  Tuned Comparison Methods}
\label{tab:cifar10_acc}
  \centering
  \resizebox{\columnwidth}{!}{%
  \begin{tabular}{llllll}
    \toprule
      & ECCO-DNN & Adam & SGD & RMSProp & AdaGrad \\
    \midrule
     Resnet18 & 92.7 & 95.9 & 94.3 & 92.8 & 92.9 \\ \midrule
     Lenet5 & 61.1 & 60.4 & 63.1 & 54.5 & 62.5 \\ \midrule
     SeNet-18 & 94.1 & 89.4 & 95.3 & 93.3 & 92.7 \\ 
    \bottomrule
  \end{tabular}%
  }
\end{table}
\fi

First, note that the models trained by ECCO-DNN achieved similar test accuracy to the comparison methods, as shown in Table \ref{tab:cifar10_acc}. ECCO-DNN follows a similar training trajectory to those of the optimally tuned optimizers as shown in \ref{sec:training_loss_curve_resnet_18}. Next, a random search was conducted with 200 hyperparameter values sampled from operating range (with $\eta \in (0,100]$). Figure \ref{fig:cifar} illustrates the classification accuracies obtained from each sample in the search. Our results indicate that the comparison methods exhibited significant variability in classification accuracy and required extensive tuning to achieve acceptable performance. In contrast, ECCO-DNN achieved near-optimal classification accuracies for $\eta$ in 4 orders of magnitude (from 0.1 to 100).\footnote{The mean and standard deviations of these experiments are listed in Tables \ref{tab:resnet_random_search}-\ref{tab:senet_random_search} in the Appendix.}
\footnote{\label{wallclock_footnote}The per-epoch wall clock time is provided in Table \ref{tab:wallclock_times}.}

We further investigate the optimizers' sensitivities to their hyperparameters by uniformly perturbing the values (as found by grid search) within a normalized $\varepsilon=0.1$ ball. The resulting classification accuracies of the models are shown in Figure \ref{fig:cifar_perturbed}, and the means and standard are shown in Tables \ref{tab:resnet_perturbed_search} - \ref{tab:senet_perturbed_search} in the Appendix. We see a dramatic range of classification accuracies for all optimizers except for ECCO-DNN, which demonstrates insensitivity to any hyperparameter selection for all three models. This experiment shows that the comparison methods require prior knowledge or significant computation, data, and time to implement effective training routines, whereas ECCO-DNN can be directly implemented with minimal tuning of its single hyperparameter.


\ifarxiv
\begin{figure*}
    \centering
    \includegraphics[width=0.99\textwidth]{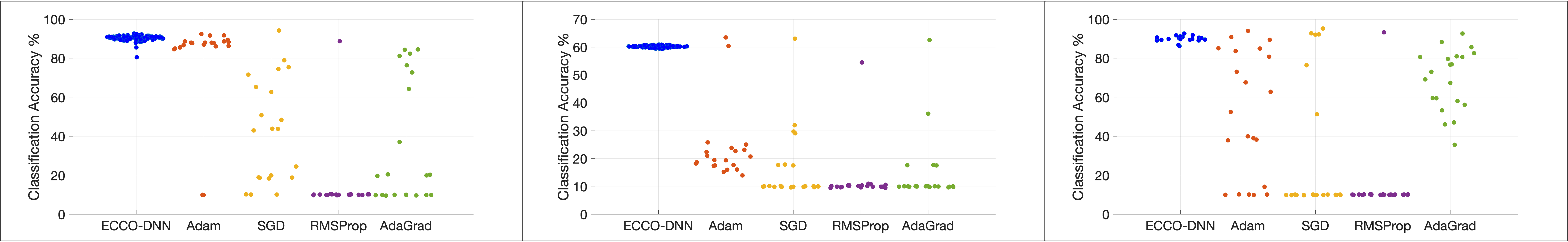}
    \caption{Hyperparameters perturbed within a normalized $\varepsilon=0.1$ ball around optimally tuned hyperparameter values for Resnet-18 (left), Lenet5 (middle) and Senet-18 (right) models.}
    \label{fig:cifar_perturbed}
\end{figure*}
\else
\begin{figure*}
    \centering
    \includegraphics[width=0.9\textwidth]{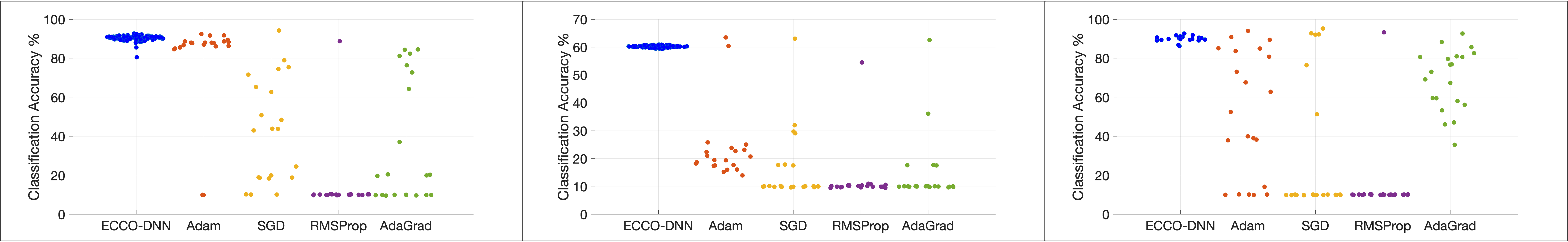}
    \caption{Hyperparameters perturbed within a normalized $\varepsilon=0.1$ ball around optimally tuned hyperparameter values for Resnet-18 (left), Lenet5 (middle) and Senet-18 (right) models.}
    \label{fig:cifar_perturbed}
\end{figure*}
\fi

\subsection{Training LSTM Model on Power Systems Dataset}
\label{sec:lstm_experiment}
In this experiment, we trained a LSTM model presented in \cite{lstm_model} to analyze a dataset of the power usage patterns of a single household \cite{power_dataset} and predict household active power consumption. As \cite{lstm_model} did not publish their hyperparameters, we first randomly sampled values within the parameter domains, trained the model for 200 epochs, and recorded the resulting mean squared error (MSE) as shown in Figure \ref{fig:lstm}. Using ECCO-DNN with  $\eta\in(0,100]$ achieved an average MSE of 0.0247 with a standard deviation of 9.2e-5. The comparison methods, however, exhibited high variance and most yield unusable results. While SGD performed the best, it still had almost 25x the variance of ECCO-DNN and 22\% of samples result in gradient overflow (i.e. NaNs). \cref{wallclock_footnote}

To verify the performance of ECCO-DNN in training the LSTM model, we employed an auto-tuner (RayTune \cite{raytune}) to search for the hyperparameters of the comparison methods that provided the lowest MSE (Tables \ref{tab:adam_optimal_hyperparameters}-\ref{tab:ec_optimal_hyperparameters} in the Appendix). As shown in Table \ref{tab:lstm_tuned_comparison}, the MSE values achieved by the tuned optimizers (MSE$=0.0244$) are within 1.2\% of the average MSE error obtained using ECCO-DNN (average MSE $=0.0247$). Furthermore, ECCO-DNN  demonstrated comparable efficiency in terms of wallclock time when compared to other optimizers, as shown in Table \ref{tab:wallclock_times}.

\ifarxiv
\begin{figure}
    \centering
\includegraphics[width=0.55\columnwidth]{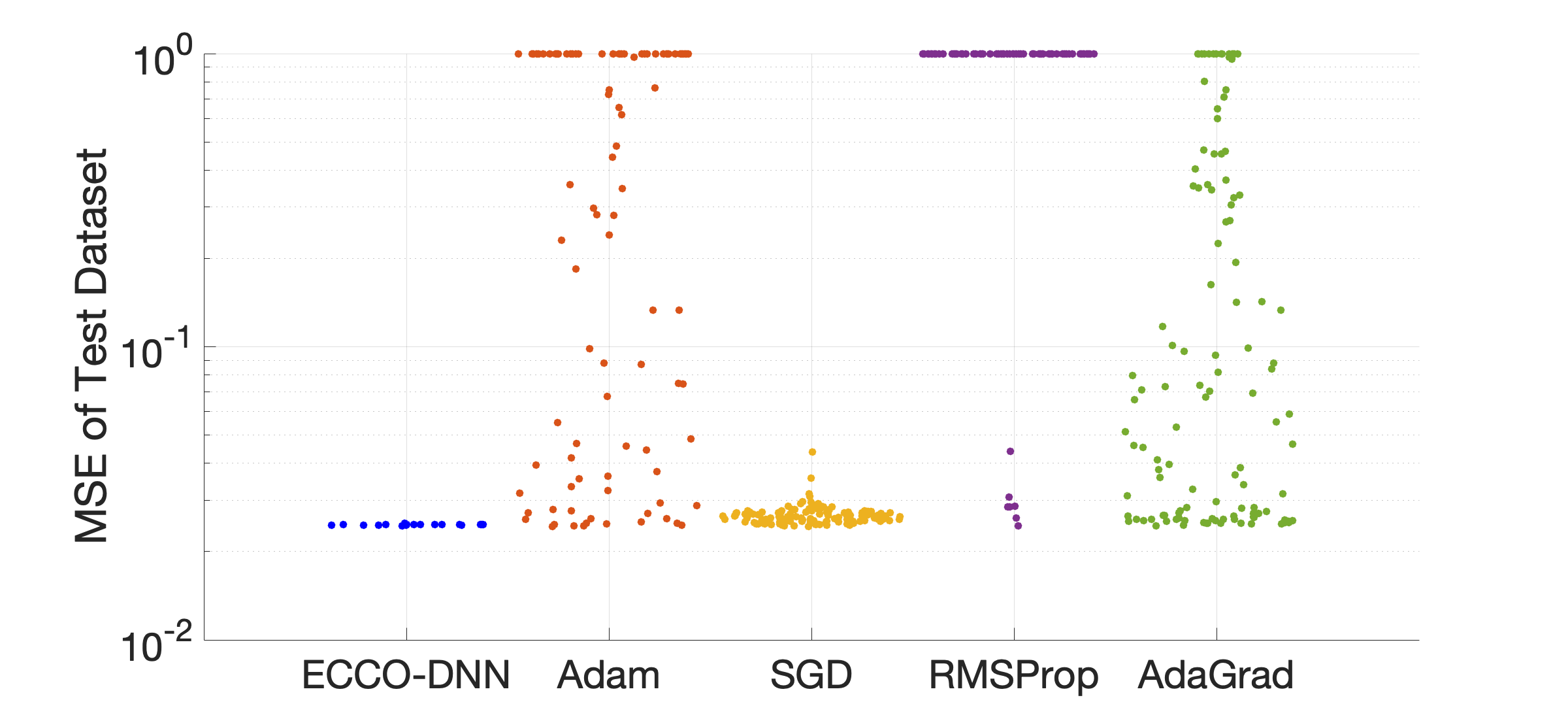}
    \caption{Random search for LSTM model \cite{lstm_model} to perform regression on \cite{power_dataset} dataset (MSE)}
    \label{fig:lstm}
\end{figure}
\else
\begin{figure}
    \centering
\includegraphics[width=0.8\columnwidth]{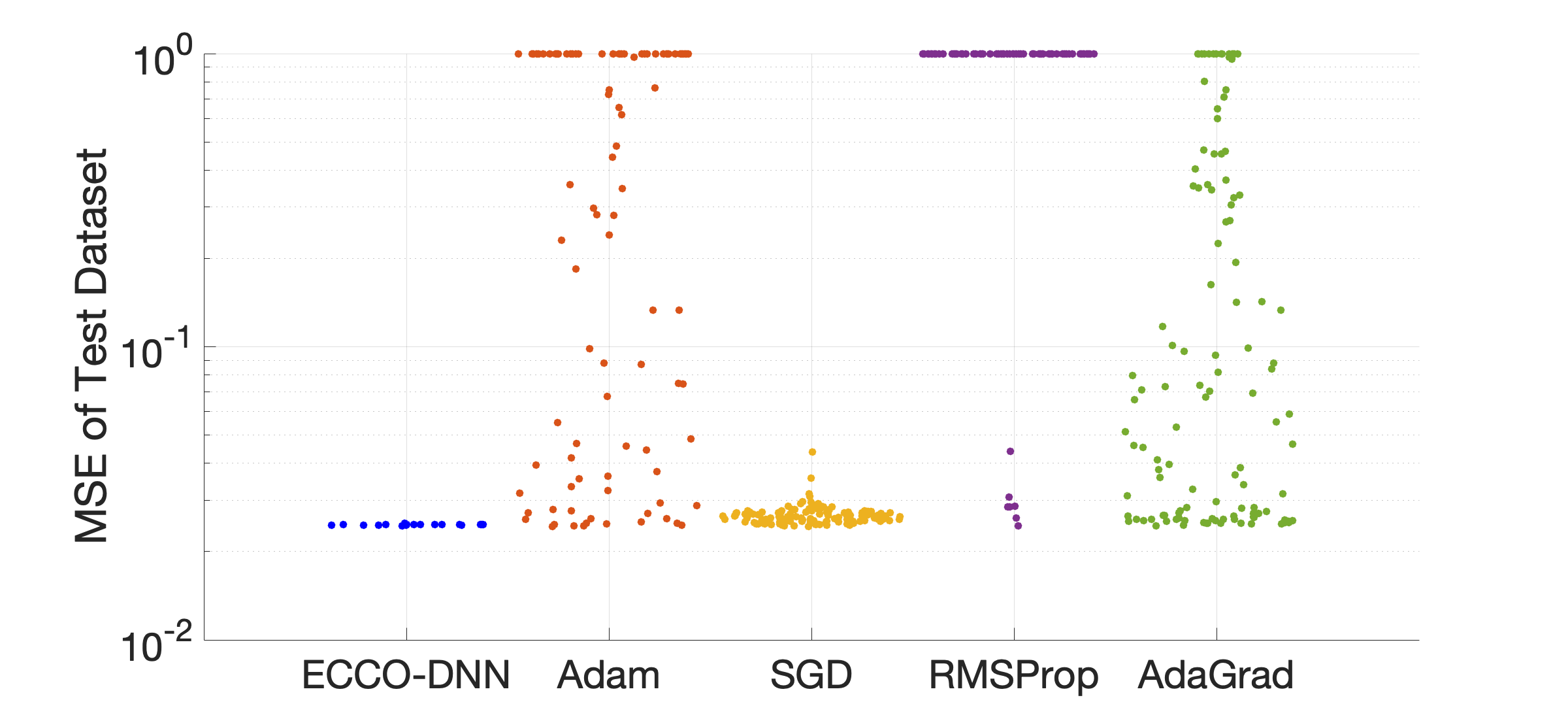}
    \caption{Random search for LSTM model \cite{lstm_model} to perform regression on \cite{power_dataset} dataset (MSE)}
    \label{fig:lstm}
\end{figure}
\fi

\ifarxiv
\begin{table}
  \caption{Percent of runs that failed or resulted in ill-fitted models during hyperparameter random search. We define an ill-fitted model as one that yields a classification accuracy $<12\%$  and a LSTM model that results in a MSE$>1$. }
  \centering
  \begin{tabular}{llllll}
    \toprule
      & ECCO-DNN & Adam & SGD & RMSProp & AdaGrad \\
    \midrule
    4-Layer DNN (MNIST) & 0 & 41.4 & 30.1 & 54.9 & 18.4 \\ \midrule
     Resnet18 (CIFAR-10) & 0 & 43.3 & 17.5 & 64.1 & 11.7 \\ \midrule
     Lenet5 (CIFAR-10) & 0 & 16.7 & 12.2 & 52.5 & 17.4 \\ \midrule
     SeNet-18 (CIFAR-10) & 0 & 60 & 57.8 & 70.2 & 11.6 \\ \midrule
     LSTM (Power) & 0 & 48.3 & 2.51 & 63.8 & 38.6 \\
    \bottomrule
  \end{tabular}%
    \label{tab:failed}
  \end{table}
\else
\begin{table}
  \caption{Percent of runs that failed or resulted in ill-fitted models during hyperparameter random search. We define an ill-fitted model as one that yields a classification accuracy $<12\%$  and a LSTM model that results in a MSE$>1$. }
  \centering
  {\tiny
  \begin{tabular}{llllll}
    \toprule
      & ECCO-DNN & Adam & SGD & RMSProp & AdaGrad \\
    \midrule
    4-Layer DNN (MNIST) & 0 & 41.4 & 30.1 & 54.9 & 18.4 \\ \midrule
     Resnet18 (CIFAR-10) & 0 & 43.3 & 17.5 & 64.1 & 11.7 \\ \midrule
     Lenet5 (CIFAR-10) & 0 & 16.7 & 12.2 & 52.5 & 17.4 \\ \midrule
     SeNet-18 (CIFAR-10) & 0 & 60 & 57.8 & 70.2 & 11.6 \\ \midrule
     LSTM (Power) & 0 & 48.3 & 2.51 & 63.8 & 38.6 \\
    \bottomrule
  \end{tabular}%
    \label{tab:failed}
  \end{table}}
  \fi

\section{Discussion}
\label{sec:discussion}
The considered experiments clearly demonstrate that ECCO-DNN can train neural networks to have comparable performance to state-of-the-art methods but without the time or data needed for hyperparameter tuning. Experiments showed that while optimally-tuned comparison methods had a minor advantage on an individual basis, ECCO-DNN's results were on par across the ensemble of methods, and were significantly faster to compute; while the comparison methods required search time for hyperparameter values, that step was unnecessary for ECCO-DNN. These observations suggest that ECCO-DNN is well-suited for applications involving rapid prototyping, time-varying data, or those sensitive to initializations. 


In Table \ref{tab:failed}, we illustrate the percentage of diverged runs in our random-search experiments. Existing optimizers often lead to a high percentage of failed runs due to ill-fitted models, defined as those with less than 12\% classification accuracy or more than 1 MSE for training the power consumption dataset. This high rate of unusable models is problematic for resource-constrained applications, such as power system analysis, where limited computational resources demand swift model development to analyze the surge of new smart meter installations.

Although current implementations of ECCO-DNN do come with an additional per-epoch runtime cost (see Table \ref{tab:wallclock_times}), we believe that this cost is outweighed by the time and resources saved by avoiding extensive hyperparameter tuning through multiple experiments. Additionally, ECCO-DNN is capable of scaling to larger models and datasets. This is demonstrated in \ref{sec:trainig_cifar_100} by training Resnet-50 and Resnet-101 models on CIFAR-100.

\section{Conclusion}

We introduce a stochastic first-order gradient-based optimizer specialized for deep neural networks that is agnostic to hyperparameters. The optimizer leverages a dynamical system model of the optimization process to shape the trajectory of the parameters and select time steps using properties of numerical integration and knowledge of the network structure. The optimizer is not sensitive to its single hyperparameter, which can vary in three orders of magnitude without affecting performance. We demonstrate this by training DNN models on datasets including MNIST, CIFAR-10 and a dataset characterizing household power consumption. ECCO-DNN exhibits comparable performance to tuned state-of-the-art optimizers Adam, SGD, RMSProp and AdaGrad. For applications requiring fast prototyping, or in which a priori knowledge on hyperparameters is unavailable, ECCO-DNN significantly reduces the time, required data, and computational requirement for hyperparameter tuning. These results suggest that ECCO-DNN is well-equipped for a practitioner's toolkit to train neural networks, and is a critical step towards hyperparameter-agnostic training.
\clearpage

\bibliography{neurips_2023}

\clearpage
\appendix
\onecolumn
\section{Derivation of \ref{z true}} \label{gf derivation}
\begin{align}
    &\max_{Z}- \frac{1}{2} \frac{d}{dt}\Vert \dfdxt \Vert^2,\\
    &=\max_{Z} - \dfdxt^{\top} \frac{d}{dt} \dfdxt, \label{passivity_max_Z}\\
    &=\max_Z \dfdxt^{\top}\nabla^2 f(\xt)Z^{-1}\dfdxt,\label{max dvdt}
\end{align}

The optimization problem \eqref{max dvdt} must yield a diagonal $Z$. To this end, we expand $\dfdxt$ to a diagonal matrix and shrink $Z^{-1}$ to a vector. Define $G(\xt)$ be a diagonal matrix where the diagonal elements are the gradient $G_{ii}(\xt) = \frac{\partial f(\xt)}{\partial \x_i(t)}$. Let $\mathbf{z}$ be a vector where $\mathbf{z}_i=Z_{ii}^{-1}$. Then \eqref{max dvdt} is equal to,
\begin{align}
    &=\max_{\z} \dfdxt^{\top}\nabla^2 f(\xt)G(\xt)\z-\frac{\delta}{2} \Vert \z\Vert^2,
\end{align}
where a regularization term with $\delta>0$ has been added for tractability. Taking the derivative $\frac{\partial}{\partial \z}$:
\begin{gather}
    G(\xt)\nabla^2f(\xt) \dfdxt -\delta \z \equiv \vec{0}.\\
    \z = \frac{1}{\delta} G(\xt) \nabla^2 f(\xt)\dfdxt.\label{final z}
\end{gather}
Any negative value of $\z_i$ occurs when $ G_{ii}(\xt)\nabla^2f(\xt) \dfdxt = Z_{ii}(\x)\dot{\x}_i(t) \frac{d}{dt} \dfdxt$ is negative. 
To ensure that $Z^{-1}$ performs as least as well as the base case where $Z^{-1}=I$ (and to retain positivity and invertibility of $Z$), any $\z_i<1$ is truncated to $\z_i=1$. For the same reasons, we set $\delta\equiv 1$ in the main paper.  

For any $\delta>0$, the final construction of the control matrix $Z(\xt)^{-1}$ is:
\begin{equation}
    Z_{ii}(\xt)^{-1} = \max\{\delta^{-1} [G(\xt) \nabla^2 f(\xt)\dfdxt]_i,1\}.
\end{equation}

This construction is a second order method as it uses Hessian information. In comparison to Newton methods, however, this $Z$ does not require a Hessian inversion step.

\section{Proof of Theorem \ref{convergence theorem}}
\label{main proof}

\begin{align}
    \frac{d}{dt}f(\xt)&=\langle \dfdxt,\dxdt(t)\rangle \\
    &= -\dfdxt^{\top} Z(\xt)^{-1}\dfdxt\\
    &\leq -d_1\Vert \dfdxt\Vert^2 \label{bounded dfdt}
\end{align}
Where \eqref{bounded dfdt} holds by Assumption \ref{z def}. 

Eq. \eqref{bounded dfdt} implies that the objective function is non-increasing along $\xt$.

For $t>0$ consider $\int_0^t \Vert\dfdxt\Vert^2dt$. By \eqref{bounded dfdt},
\begin{align}
    \int_0^t\Vert\dfdxt\Vert^2dt&\leq\frac{1}{d_1}(f(\x(0))-f(\x(t)))\\
    &\leq \frac{2R}{d_1} \label{bounded int norm}
\end{align}
Where \eqref{bounded int norm} holds by Assumption \ref{a1}.

Then for all $t>0$,
\begin{equation}
     \int_0^t\Vert\dfdxt\Vert^2dt\leq \frac{2R}{d_1}<\infty
\end{equation}

For all $\x$, $\y\in \R^n$,
\begin{align}
    \big \vert\Vert \dfdx \Vert^2 - \Vert \dfdy \Vert^2 \big \vert &\leq \big \vert \Vert \dfdx \Vert + \Vert \dfdy \Vert \big \vert \ \big \vert \Vert \dfdx \Vert -\Vert \dfdy \Vert \big \vert\\
    &\leq 2B \big \vert \Vert \dfdx \Vert - \Vert \dfdy \Vert \big \vert\\
    &\leq 2B \Vert \dfdx - \dfdy\Vert\\
    &\leq 2BL \Vert \x-\y \Vert
\end{align}

Therefore we can conclude that $\Vert \dfdx\Vert^2:\R^n\to\R$ is Lipschitz and hence uniformly continuous.

By Assumptions \ref{a4} and \ref{z def}, $\Vert \dxdt(t)\Vert$ is bounded and so $\xt$ is uniformly continuous in $t$.

Therefore, the composition $\Vert \dfdxt \Vert^2: \R_+\to\R$ is uniformly continuous in $t$. 

Since $\int_0^t\Vert\dfdxt\Vert^2dt<\infty$ and $\Vert\dfdxt\Vert^2$ is a uniformly continuous function of $t$, we conclude that $\lim_{t\to\infty}\Vert\dfdxt\Vert=0$.

\section{Proof of Lemma \ref{bounded z true}}\label{zbounded}

\begin{proof}
Note that the boundedness condition in Assumption \ref{z def} will hold if $0< \sum_{i=1}^n Z_{ii}(\xt)^{-1}<\infty$.

Lower bound:
\begin{align}
    Z_{ii}^{-1}(\xt) = \max\{\delta^{-1} [G(\xt) \nabla^2 f(\xt)\dfdxt]_i,1\}\geq 1>0.
\end{align}
Upper bound:
\begin{align}
     \sum_{i=1}^n Z_{ii}(\xt)^{-1} &\leq \mathbf{1}^{\top}\Big(\delta^{-1}G(\xt)\nabla^2f(\xt)\dfdxt\Big)+\sum_{i=1}^n 1\\
     &=\delta^{-1}\mathbf{1}^{\top}G(\xt)\nabla^2f(\xt)\dfdx +n\\
     &=\delta^{-1}\dfdxt^{\top}\nabla^2f(\xt)\dfdx +n\\
     &\leq \delta^{-1}\lambda_{\max}\big(\nabla^2 f(\xt)\big)\Vert \dfdxt\Vert^2+n\label{bound w lambda}\\
     &\leq \delta^{-1}\lambda_{\max}\big(\nabla^2 f(\xt)\big)B^2+n\\
     &\leq \delta^{-1}B^2\lambda_{\sup}+n<\infty
\end{align}
Where $\sup\lambda_{\max}$ exists and is finite under \ref{a3}.
\end{proof}

\section{Derivation of \eqref{z approx}}\label{z approx deriv}

The proposed control scheme may be used directly as in \eqref{z true}; however, it requires computation of the full Hessian. For applications in which the Hessian in unavailable or expensive to compute, such as machine learning, we present an approximation that may be computed from only gradient information.

We form $\widehat{Z}$ by approximating the trajectory of the optimization variable at a specific $t$ and solving for $Z$ as a function of $\ahat$; we do not estimate the Hessian matrix, and therefore $\widehat{Z}$ is not a quasi-Newton method. The proposed iterative update using \eqref{z approx} is a scaled gradient or step-size normalized descent method, similar in spirit to normalized gradient flow (in continuous time) and step-size or momentum scaling methods methods such as Adam, RMSprop, and Adagrad.

Recall the approximation of $\frac{d}{dt}\dfdxt$ as in \eqref{ahat} and then,
\begin{align}
    \ahat(\xt)&= \frac{\dfdxt - \nabla f(\x (t-\Delta t))}{\Delta t}\nonumber\\
    &\approx \frac{d}{dt}\dfdxt\\
    &=\frac{d\dfdxt}{d\xt} \frac{d\xt}{dt}\\
        &=\nabla^2 f(\xt)(-Z(\xt)^{-1}\dfdxt)\\
        &=-\nabla^2 f(\xt) Z(\xt)^{-1}\dfdx.
\end{align} 
Pre-multiply by $\delta^{-1}G(\x(t))$:
\begin{align}
    \delta^{-1}G(\xt)\ahat(\xt)\approx -\delta^{-1}G(\xt) \nabla^2 f(\xt) Z(\xt)^{-1} \dfdxt. \label{deriv 1}
\end{align}
As $Z$ is a diagonal matrix, the expression on the RHS of \eqref{deriv 1} can be rewritten as,
\begin{align}
    &-\delta^{-1}G(\xt) \nabla^2 f(\xt) Z(\xt)^{-1} \dfdxt\nonumber\\
    &=- \delta^{-1}G(\xt)\nabla^2 f(\xt)\begin{bmatrix}
    Z_{11}(\xt)^{-1}\frac{\partial f(\xt)}{\partial \x_1}\\ \vdots\\ Z_{nn}(\xt)^{-1}\frac{\partial f(\xt)}{\partial \x_n}
    \end{bmatrix} \\
     &=-\delta^{-1}G(\xt)\nabla^2 f(\xt)\begin{bmatrix}
        \frac{\partial f(\xt)}{\partial x_1} & 0 &\dots & 0\\
        0  & \frac{\partial f(\xt)}{\partial x_2} & \ddots & 0\\
        \vdots & \vdots & \ddots & \vdots \\
        0 & 0 & \dots & \frac{\partial f(\xt)}{\partial x_n}
    \end{bmatrix}\begin{bmatrix}
        Z_{11}(\xt)^{-1}\\ \vdots \\ Z_{nn}(\xt)^{-1}
    \end{bmatrix}\\
    &=-\delta^{-1}G(\xt)\nabla^2 f(\xt)G(\xt)\z(\xt). \label{deriv 4}
\end{align}
In comparison, the objective is to compute:
\begin{align}
    \z(\xt)&= \delta^{-1}G(\xt) \nabla^2 f(\xt) \dfdxt \label{deriv 2}\\
    &= \delta^{-1}G(\xt) \nabla^2 f(\xt)G(\xt) \mathbf{1}\label{deriv 3}.
\end{align}
Where \eqref{deriv 2} follows from \eqref{final z} and \eqref{deriv 3} holds because by definition $G(\xt)_{ii} = [\dfdxt]_i$.

Define the shorthand notation for the expressions $\Ab\triangleq G(\xt)\nabla^2 f(\xt)G(\xt)$ and $\z\triangleq \z(\xt)$. Then \eqref{deriv 4} becomes $-\delta^{-1}\Ab\z$ and similarly \eqref{deriv 3} becomes $\delta^{-1}\Ab\mathbf{1}$. From \eqref{deriv 1} and \eqref{deriv 4},
\begin{equation}
    -\delta^{-1}\Ab \z \approx \delta^{-1}G(\xt)\ahat(\xt) \label{deriv 6}
\end{equation}
Use the definition of $\z$ in \eqref{deriv 3} and substitute into \eqref{deriv 6}:
\begin{align}
    \z &= \delta^{-1}\Ab \mathbf{1} \Rightarrow -\delta^{-1}\Ab \z = -\delta^{-2}\Ab\Ab\mathbf{1} \approx   \delta^{-1}G(\xt)\ahat(\xt)
\end{align}
Pre-multiply by $-\mathbf{1}^{\top}$ and note that $\Ab = \Ab^{\top}$ :
\begin{align}
    (\delta^{-1}\Ab\mathbf{1})^{\top}(\delta^{-1}\Ab\mathbf{1}) &\approx -\delta^{-1}\mathbf{1}^{\top}G(\xt)\ahat(\xt)\\
    \z^{\top}\z&\approx-\delta^{-1}\mathbf{1}^{\top}G(\xt)\ahat(\xt) \label{deriv 5}
\end{align}

One possible solution to \eqref{deriv 5} is thus,
\begin{align}
    &\z_i^2(\xt)\triangleq -[\delta^{-1}G(\xt)\ahat(\xt)]_i\\
    &\Rightarrow \z_i(\xt) = \sqrt{\max\{-[\delta^{-1}G(\xt)\ahat(\xt)]_i, 0\}}.
\end{align}

Similarly to the derivation of \eqref{z true}, we truncate $\z_i(\xt)\geq 1$ to finish the derivation.

\section{Proof of Lemma \ref{bounded z approx}} \label{zhat bounded}

\begin{proof}
Note that the boundedness condition in Assumption \ref{z def} will hold if $0< \sum_{i=1}^n \hat{Z}_{ii}(\xt)^{-1}<\infty$.

Lower bound:
\begin{align}
    \hat{Z}_{ii}(\xt)^{-1}=\max\{\sqrt{-\delta^{-1}[G(\xt)\ahat(\xt)]_i},1\}\geq 1 >0.
\end{align}
Upper bound:
\begin{align}
    \hat{Z}_{ii}(\xt)^{-1} &\leq \max\{\sqrt{-\delta^{-1}[G(\xt)\ahat(\xt)]_i},1\}^2\\
    \sum_{i=1}^n \hat{Z}_{ii}(\xt)^{-1} &\leq  \sum_{i=1}^n -\delta^{-1}[G(\xt)\ahat(\xt)]_i+\sum_{i=1}^n 1\\
    &= \delta^{-1}(\Delta t)^{-1}\Big(\dfdxt^{\top} \nabla f(\x(t-\Delta t)) -\dfdxt^{\top}\dfdxt \Big)+n\\
    &=\delta^{-1}(\Delta t)^{-1}\dfdxt^{\top}\Big( \nabla f(\x(t-\Delta t)) -\dfdxt \Big)+n\\
    &\leq\delta^{-1}(\Delta t)^{-1}\left\vert\dfdxt^{\top}\Big(  \nabla f(\x(t-\Delta t)) -\dfdxt \Big)\right\vert+n\\
    &\leq \delta^{-1}(\Delta t)^{-1}\Vert \dfdxt\Vert \Vert \nabla f(\x(t-\Delta t)) -\dfdxt\Vert+n\\
    &\leq \delta^{-1}(\Delta t)^{-1} BL\Vert \x(t-\Delta t)-\x(t)\Vert+n\\
    &\leq \delta^{-1}(\Delta t)^{-1} BL\max_{\x^*\in S}\Vert \x(0)-\x^*\Vert+n<  \infty
\end{align}
Where the last line holds because $\Vert \x(0)-\x^*\Vert$ is bounded as $\x(0)$ is given to be finite and $\x^*$ is finite due to \ref{a2}.
\end{proof}

\section{Computation Complexity}
\subsection{Proof of Lemma \ref{z true complexity}} \label{complexity full}
\begin{lemma} \label{z true complexity}
    Given the gradient and Hessian, the computation complexity to evaluate $Z^{-1}$ at a specific $\xt$ is $\mathcal{O}(n^2)$.
\end{lemma}
\begin{proof}

Given the gradient and Hessian, the per-iteration computation complexity to evaluate $Z$ according to \eqref{z true} is $\mathcal{O}(n^2)$. This can be verified as each element $i$ requires $\mathcal{O}(n)$ computations:
\begin{equation}
    Z_{ii}^{-1}(\xt) = \max\left\{\delta^{-1} \frac{\partial f(\xt)}{\partial \x_i(t)}\sum_{j=1}^n \frac{\partial^2 f(\xt)}{\partial \x_i \partial \x_j}\frac{\partial f(\xt)}{\partial \x_j}, 1\right\}
\end{equation}
\end{proof}

\subsection{Proof of Lemma \ref{z approx complexity}} \label{complexity approx}
Given the current and previous gradient, the per-iteration computation complexity to evaluate $\widehat{Z}$ according to \eqref{z approx} is $\mathcal{O}(n)$. This can be verified as each element $i$ requires $\mathcal{O}(1)$ computations:
\begin{equation}
    \hat{Z}_{ii}^{-1}(\xt) = \max\left\{-\delta^{-1}\frac{1}{\Delta t} \frac{\partial f(\xt)}{\partial \x_i(t)} \left(\frac{\partial f(\xt)}{\partial \x_i}-\frac{\partial f(\x(t-\Delta t))}{\partial \x_i}\right), 1\right\}
\end{equation}

\section{Equivalent Circuit Model}
\label{sec:equivalent-circuit}
As stated, the ECCO-DNN algorithm was developed from dynamical systems model of the optimization trajectory.
This model was used to help derive and refine the control matrix $Z$ of the algorithm. We specifically use the domain of electrical circuits to model the gradient flow, as the circuit model provides physics-based insights to facilitate the construction of $Z$.

To demonstrate that the scaled gradient flow is equivalent to a circuit, we present a circuit whose transient response is equivalent to \eqref{scaled gradient flow}. The equivalent circuit (EC) is composed of $n$ sub-circuits, with each sub-circuit representing the transient waveform of a single variable, $\x_i(t)$. The diagram of a single sub-circuit is shown in Figure \ref{fig:ECorig}. Each sub-circuit is composed of two elements: a nonlinear capacitor (on the left) and a voltage controlled current source (VCCS) represented by the element on the right. Note that for multi-dimensional $\x$, the sub-circuits are coupled via the VCCSs. 

The voltage across the capacitor is defined to be $\x_i(t)$, and the capacitance of the capacitor is likewise defined as $Z_{ii}(\xt)$. Based on the voltage-current relationship of a capacitor, the current will be equal to $I_i^c(\x(t))=Z_{ii}(\xt)\dxdt_i(t)$ where $c$ labels attachment to the capacitor and $i$ indexes the element of $\x$. From Kirchhoff's Current Laws (KCL) \cite{desoer2010basic}, the capacitor current must be equal to the negative of the current produced by the VCCS; therefore $Z(\xt)\dot{\x}(t)=-\dfdxt$ which is identical to \eqref{scaled gradient flow}. If the circuit reaches steady-state at some time $t'$, the capacitor no longer produces a current ($I_c(\x(t))_{t\geq t'}=0$) and therefore the node voltage $\x(t)_{t\geq t'}$ remains stationary. By KCL, the VCCS elements also produce zero current at steady-state ($\dfdxt_{t\geq t'}=\vec{0}$), implying that we have reached a point where $\dfdxt=\vec{0}$, which is defined as a critical point of the objective function. 




\begin{figure}
\centering
\begin{minipage}[t]{.48\linewidth}
  \centering
    \begin{circuitikz}[scale=0.45]
    \ctikzset{label/align = rotate}
    \draw
    (0,0) to[C=\small $Z_{ii}(\xt)$] (0,3)
      to[short] (3,3)
      to[cI, l=\small $\frac{\partial f(\x(t))}{\partial \x_i}$, label distance=3pt] (3,0)
      to[short] (0,0)
    ;
    \draw (3,3) to [short,-o] (4,3) node[above]{\small $\x_i(t)$};
    \draw (1.5,0) node[ground] (){};
    \end{circuitikz}
    \caption{\small{Equivalent Circuit Model of \eqref{scaled gradient flow}} }
    \label{fig:ECorig}
\end{minipage}%
\hfill
\begin{minipage}[t]{.48\linewidth}
  \centering
   \begin{circuitikz}[scale = 0.45]
   \ctikzset{label/align = rotate}
   \draw
    (0,0) to[C, i^<= \small\color{white} $0$, ] (0,3)
      to[short] (3,3)
      to[cI, l=\small $\quad \frac{d}{dt} \frac{\partial f(\x(t))}{\partial \x_i}$, label distance=3pt] (3,0)
      to[short] (0,0)
    ;
    \draw (3,3) to [short,-o] (4,3) node[above]{\small $\dot\x_i(t)$};
    \draw (1.5,0) node[ground] (){};
    \draw (-0.5,3.8) circle [radius=0] node {\small $\frac{d}{dt}(Z(\xt)\dot\x(t))_i$};
    \end{circuitikz}
    \caption{\small{Adjoint Equivalent Circuit Model of \eqref{time deriv}}}
    \label{fig:ECadj}
\end{minipage}
\end{figure}

To gain insight into the energy transfer in the equivalent circuit model, we also construct a circuit representation for the behavior of $\dxdt(t)$. Taking the time derivative of \eqref{scaled gradient flow}, we construct a circuit representation called the \emph{adjoint circuit}, as shown in Figure \ref{fig:ECadj}. Similar to the equivalent circuit model, each adjoint sub-circuit is composed of a capacitor and a VCCS element, with the node voltage now representing $\dot{\x}_i(t)$. The adjoint capacitor has a current of $ \bar{I}_i^c(\x(t)) = \frac{d}{dt}[Z(\xt) \dxdt(t)]_i$ where:
\small
\begin{align}
    \frac{d}{dt}\left(Z(\xt)\dot{\x}(t)\right) &= - \frac{d}{dt} \dfdxt = -\nabla^2 f(\xt) \frac{d \x(t)}{dt},\nonumber\\
    &= \nabla^2 f(\xt) Z(\xt)^{-1} \dfdxt.\label{time deriv}
\end{align}
\normalsize


The energy of the adjoint circuit is analyzed to provide intuition on controlling the circuit; when the adjoint circuit is at steady-state then $\dot{\x}(t)=0$, meaning that the original circuit is also at a steady-state.  The capacitor in the adjoint circuit is initially charged to  $\bar{Q}_i^c(\x(0))$ and discharges to reach steady-state. The energy stored in the capacitor is proportional to the charge of the capacitor, which in turn is:
\begin{equation}
    \bar{Q}_i^c(\xt) = \int_0^t\bar{I}_i^c(\x(t))dt = Z_{ii}(\xt)\dot{\x}_i(t). \label{charge}
\end{equation}
Let $\bar{Q}_c(\xt) = [\bar{Q}^c_1(\xt), \dots, \bar{Q}^c_n(\xt)]^{\top}$.


Any component-wise scaled gradient flow problem \eqref{scaled gradient flow} satisfying \ref{a1} - \ref{z def} can thus be modeled as an EC and adjoint EC. 

We can leverage the circuit formulation to construct a control policy for the scaling matrix, $Z^{-1}$. In the circuit sense, fast convergence to a critical point of the optimization problem is equivalent to fast convergence to a stable steady-state of the dynamical system. Thus, we must choose the nonlinear capacitances $Z$ as a function of $\xt$ to discharge the capacitors as quickly as possible. 

The squared charge stored in the adjoint circuit capacitor at some time $t$ is equal to $\Vert \bar{Q}_c(\xt)\Vert^2$, which by \eqref{charge} is equal to  $\Vert Z(\xt)\dot{\x}(t))\Vert^2$ and by \eqref{scaled gradient flow} is equal to $\Vert \nabla f(\xt)\Vert^2$. Thus to quickly \emph{dissipate} charge, the following optimization problem can be defined, which maximizes the negative time gradient of the charge:

\vspace{-0.5cm}
\small
\begin{align}
    &\max_{Z}- \frac{d}{dt}\Vert \bar{Q}_c(\xt) \Vert^2=\max_{Z}- \frac{d}{dt}\Vert \dfdxt \Vert^2
\end{align}
\normalsize

This formulation serves as a physics-based argument for the construction of $Z$ in \eqref{z criteria}, which also matches the optimization-based logic that the gradient of the objective should move towards zero as time progresses.

\section{Hyperparameter Selection}

The hyperparameters and their respective operating ranges for each optimizer (Adam, SGD, AdaGrad, RMSProp, and ECCO-DNN) are provided in Tables \ref{tab:adam_hyperparameters}-\ref{tab:ec_hyperparameters}. Note, Adam and SGD are paired with a Cosine Annealing scheduler \cite{scheduler_cosine} whose maximum number of iterations, $T_{max}$, is part of the hyperparameter set. 

\begin{table}[H]
  \caption{Adam Hyperparameters and the Operating Range}
  \label{tab:adam_hyperparameters}
  \centering
  \begin{tabular}{lll}
    \midrule
    Hyperparameter     & Infimum     & Maximum \\
    \midrule
    Initial Learning Rate & 0 & 1 \\
    $\beta_1$     & 0.5 & 1      \\
    $\beta_2$ & 0.5 & 1  \\
    Scheduler $T_{max}$ & 0 & 10,000 \\
    \bottomrule
  \end{tabular}
\end{table}

\begin{table}[H]
  \caption{SGD Hyperparameters and the Operating Ranges}
  \label{tab:sgd_hyperparameters}
  \centering
  \begin{tabular}{lll}
    \midrule
    Hyperparameter     & Infimum     & Maximum \\
    \midrule
    Initial Learning Rate & 0 & 1 \\
    Weight Decay & 0 & 1 \\
    Scheduler $T_{max}$ & 0 & 10,000 \\
    \bottomrule
  \end{tabular}
\end{table}

\begin{table}[H]
  \caption{AdaGrad Hyperparameters and the Operating Ranges}
  \label{tab:adagrad_hyperparameters}
  \centering
  \begin{tabular}{lll}
    \midrule
    Hyperparameter     & Infimum     & Maximum \\
    \midrule
    Learning Rate & 0 & 1 \\
    Weight Decay & 0 & 1 \\
    Learning Rate Decay & 0 & 1 \\
    \bottomrule
  \end{tabular}
\end{table}

\begin{table}[H]
  \caption{RMSProp Hyperparameters and the Operating Ranges}
  \label{tab:rmsprop_hyperparameters}
  \centering
  \begin{tabular}{lll}
    \midrule
    Hyperparameter     & Infimum     & Maximum \\
    \midrule
    Learning Rate & 0 & 1 \\
    Weight Decay & 0 & 1 \\
    $\alpha$ & 0 & 1 \\
    \bottomrule
  \end{tabular}
\end{table}

\begin{table}[H]
  \caption{ECCO-DNN Hyperparameters and the Operating Ranges}
  \label{tab:ec_hyperparameters}
  \centering
  \begin{tabular}{lll}
    \midrule
    Hyperparameter     & Infimum     & Maximum \\
    \midrule
    LTE Tolerance ($\eta$) & 0 & 100 \\
    \bottomrule
  \end{tabular}
\end{table}

\subsection{Mean and Standard Deviations of Simulations}

The search for the optimal hyperparameters for each optimizer was accomplished by performing a random search. In this section, we provide mean and standard deviation of the classification accuracies for random-search for each simulation in Section \ref{sec:results}.


\subsubsection{Mean and Standard Deviations for Classifying MNIST in Section \ref{sec:mnist_experiment}}

We first studied the sensitivity to hyperparameters by perturbing the hyperparameters of each optimizer within a $\varepsilon=0.1$ ball to train the 4-layer DNN in \cite{mnist_model} to classify MNIST dataset. The mean and standard deviation of the classification accuracies from the random-search for each optimizer is provided in Table \ref{tab:mnist_random_search}.

Observe that ECCO-DNN with randomly chosen hyperparameters achieves a classification accuracy as expected by an optimally tuned state-of-the-art method. The low standard deviation suggests that this finding is not the result of a random good seed. In comparison, the other optimizers show a high level of sensitivity to perturbations of the hyperparameter as we observe low mean classification accuracies with a high standard deviation.

\begin{table}[H]
  \caption{Mean and Standard Deviation (Std.) of classification accuracies on test set for MNIST experiment (Section \ref{sec:mnist_experiment}) subject to random hyperparameter search for N=200}
  \label{tab:mnist_random_search}
  \centering
  \begin{tabular}{lll}
    \midrule
         & Mean     & Std. \\
    \midrule
    Adam & 28.63 & 27.67 \\
    AdaGrad & 33.30 & 24.82 \\
    SGD & 67.22 & 41.36 \\
    RMSProp & 40.73 & 37.86 \\
    ECCO-DNN & 99.73 & 0.07 \\
    \bottomrule
  \end{tabular}
\end{table}

\subsubsection{Mean and Standard Deviations for Classifying CIFAR-10 in Section \ref{sec:cifar_experiment}}
\begin{figure*}
    \centering
    \includegraphics[width=0.9\textwidth]{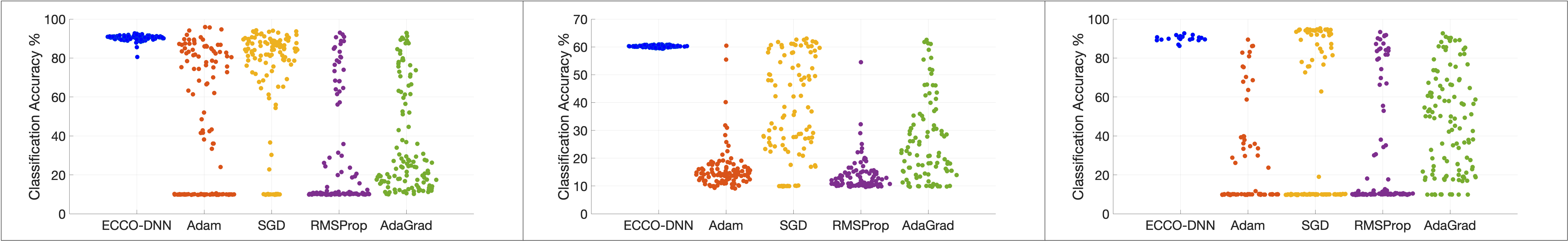}
    \caption{Random Search for Hyperparameters within operating range ball for Resnet-18 (left), Lenet5 (middle) and Senet-18 (right) models.}
    \label{fig:cifar}
\end{figure*}

We train three DNN models (Resnet-18 \cite{resnet}, Lenet-5 \cite{lenet}, and Senet-18 \cite{senet}) to classify CIFAR-10 dataset. To find the best hyperparameter selections for each optimizer, we perform a random search over the entire hyperparamter operating range. The mean and standard deviation of the classification accuracies are provided in Tables \ref{tab:resnet_random_search}-\ref{tab:senet_random_search}.

In all three models (Resnet-18, Lenet-5, and Senet-18), ECCO-DNN consistently achieves classification accuracies that are within a 5\% margin to the accuracies obtained using the optimally tuned optimizers described in Table \ref{tab:cifar10_acc}. Even when subject to a random hyperparameter search spanning the entire operation region, ECCO-DNN consistently achieves classification accuracies with a low standard deviation. This signifies that we can obtain an acceptable performance with ECCO-DNN without the need for hyperparameter tuning $\eta$. Conversely, Adam, SGD, AdaGrad, and RMSProp exhibit considerably high standard deviations in their final classification accuracies, underscoring the crucial importance of selecting appropriate hyperparameter values for optimal performance.

\begin{table}[H]
  \caption{
  Mean and Standard Deviation (Std.) of classification accuracies on test set for CIFAR-10 experiment using Resnet-10, subject to random hyperparameter search for N=200.}
  \label{tab:resnet_random_search}
  \centering
  \begin{tabular}{lll}
    \midrule
         & Mean     & Std. \\
    \midrule
    Adam & 46.52 & 34.83 \\
    AdaGrad & 36.34 & 27.25 \\
    SGD & 68.93 & 29.57 \\
    RMSProp & 27.28 & 28.34 \\
    ECCO-DNN & 90.28 & 1.44 \\
    \bottomrule
  \end{tabular}
\end{table}

\begin{table}[H]
  \caption{Mean and Standard Deviation (Std.) of classification accuracies on test set for CIFAR-10 experiment using Lenet-5, subject to random hyperparameter search for N=200.}
  \label{tab:lenet_random_search}
  \centering
  \begin{tabular}{lll}
    \midrule
         & Mean     & Std. \\
    \midrule
    Adam & 15.87 & 7.53 \\
    AdaGrad & 26.15 & 14.56 \\
    SGD & 38.17 & 17.66 \\
    RMSProp & 13.63 & 5.91 \\
    ECCO-DNN & 60.24 & 0.378 \\
    \bottomrule
  \end{tabular}
\end{table}

\begin{table}[H]
  \caption{Mean and Standard Deviation (Std.) of classification accuracies on test set for CIFAR-10 experiment using Senet-10, subject to random hyperparameter search for N=200.}
  \label{tab:senet_random_search}
  \centering
  \begin{tabular}{lll}
    \midrule
         & Mean     & Std. \\
    \midrule
    Adam & 26.98 & 25.74 \\
    AdaGrad & 45.29 & 25.86 \\
    SGD & 42.77 & 39.45 \\
    RMSProp & 26.66 & 29.37 \\
    ECCO-DNN & 90.01 & 1.57 \\
    \bottomrule
  \end{tabular}
\end{table}

We then study the sensitivity of the classification accuracy to the hyperparameters by perturbing the hyperparameters within a $\varepsilon=0.1$ ball around the optimal hyperparameters found in the random-search. Fifty hyperparamter values were randomly selected within the uniform distribution in Section \ref{sec:cifar_experiment}and the mean and standard deviation of the classification accuracies from the perturbed random search is shown in Tables \ref{tab:resnet_perturbed_search}-\ref{tab:senet_perturbed_search}.

We observe that Adam, AdaGrad, SGD, and RMSProp are highly sensitive to the hyperparameter selections, as a perturbation within $\varepsilon=0.1$ ball can cause a high variance in standard deviations. These findings highlight the importance of carefully choosing hyperparameters for these optimizers.  On the other hand, ECCO-DNN demonstrates a notable insensitivity to the value of $\eta$, whereby small perturbations around the optimal value generally have minimal impact on the overall performance of the model.

\begin{table}[H]
  \caption{Mean and Standard Deviation (Std.) of classification accuracies on test set for CIFAR-10 experiment using Resnet-10, subject to random hyperparameter search within $\varepsilon=0.1$ ball of the tuned hyperparameters for N=200.}
  \label{tab:resnet_perturbed_search}
  \centering
  \begin{tabular}{lll}
    \midrule
         & Mean     & Std. \\
    \midrule
    Adam & 80.57 & 23.58 \\
    AdaGrad & 35.78 & 31.50 \\
    SGD & 42.95 & 26.71 \\
    RMSProp & 13.93 & 17.62 \\
    ECCO-DNN & 90.28 & 1.44 \\
    \bottomrule
  \end{tabular}
\end{table}

\begin{table}[H]
  \caption{Mean and Standard Deviation (Std.) of classification accuracies on test set for CIFAR-10 experiment using Lenet-5, subject to random hyperparameter search within $\varepsilon=0.1$ ball of the tuned hyperparameters for N=200.}
  \label{tab:lenet_perturbed_search}
  \centering
  \begin{tabular}{lll}
    \midrule
         & Mean     & Std. \\
    \midrule
    Adam & 23.71 & 13.13 \\
    AdaGrad & 14.77 & 12.52 \\
    SGD & 16.46 & 12.93 \\
    RMSProp & 12.19 & 9.69 \\
    ECCO-DNN & 60.24 & 0.378 \\
    \bottomrule
  \end{tabular}
\end{table}

\begin{table}[H]
  \caption{Mean and Standard Deviation (Std.) of classification accuracies on test set for CIFAR-10 experiment using Senet-10, subject to random hyperparameter search within $\varepsilon=0.1$ ball of the tuned hyperparameters for N=200.}
  \label{tab:senet_perturbed_search}
  \centering
  \begin{tabular}{lll}
    \midrule
         & Mean     & Std. \\
    \midrule
    Adam & 51.64 & 31.95 \\
    AdaGrad & 69.05 & 15.65 \\
    SGD & 30.92 & 35.05 \\
    RMSProp & 14.01 & 18.18 \\
    ECCO-DNN & 90.01 & 1.57 \\
    \bottomrule
  \end{tabular}
\end{table}

\subsubsection{Mean and Standard Deviations for Training LSTM Model on Power Systems Dataset in Section \ref{sec:lstm_experiment}}

In Section \ref{sec:lstm_experiment}, we trained an LSTM model \cite{lstm_model} to predict the household power consumption using the dataset provided in \cite{power_dataset}. To train the LSTM model, we performed a random-search for the optimal hyperparamter values within the respective operating range for each optimizer. The mean and standard deviation of the MSE for random-search is provided in Table \ref{tab:lstm_random_search}.

This experiment demonstrates the process of hyperparameter tuning for a newly published model and dataset.  By conducting a random hyperparameter search across the entire operating range, we observe that Adam, Adagrad and RMSProp yield unusable results, as their final mean MSEs are orders of magnitude higher than their optimal MSE of 0.0244, presented in Table \ref{tab:lstm_tuned_comparison}. Although SGD produces a mean MSE comparable to its optimal value, it exhibits a substantially larger standard deviation, approximately 20 times greater than that of ECCO-DNN.

On the other hand, ECCO-DNN achieves an average MSE that is within a 2\% difference from the optimal MSE obtained by the other optimizers. Furthermore, a low standard deviation of $9.2e-5$ indicates that we can reliably use ECCO-DNN to train the LSTM model with near optimal performance.

\begin{table}[H]
  \caption{Mean and Standard Deviation (Std.) of Mean Squared Errors on test set for power systems experiment using LSTM model \cite{lstm_model}, subject to random hyperparameter search  for N=200.}
  \label{tab:lstm_random_search}
  \centering
  \begin{tabular}{lll}
    \midrule
         & Mean     & Std. \\
    \midrule
    Adam & 7.85e22 & 6.65e23 \\
    AdaGrad & 0.415 & 0.843 \\
    SGD & 0.0267 & 0.0022 \\
    RMSProp & 130.56 & 124.62 \\
    ECCO-DNN & 0.0247 & 9.2e-5 \\
    \bottomrule
  \end{tabular}
\end{table}

\subsection{Best Hyperparameter Values}

The optimal hyperparameters from the random-search for each optimizer (Adam, SGD, AdaGrad, RMSProp, and ECCO-DNN) for the experiments in \ref{sec:results} are provided in Table \ref{tab:adam_optimal_hyperparameters}-\ref{tab:ec_optimal_hyperparameters}.

\begin{table}[H]
  \caption{Optimally tuned hyperparameters for Adam as found by random search}
\label{tab:adam_optimal_hyperparameters}
  \centering
  \begin{tabular}{llllll}
    \midrule
    Hyperparameter     & MNIST Classifier & Resnet-18 & Lenet-5 & Senet-18 & LSTM \\
    \midrule
    Initial Learning Rate & 0.000726 & 0.0495 & 0.00107 & 0.904 & 0.000295 \\
    $\beta_1$ & 0.841 & 0.732 & 0.853 & 0.904 &0.895  \\
    $\beta_2$ & 0.853 & 0.626 & 0.978 & 0.975 & 0.502 \\
    Scheduler $T_{max}$ & 226 & 123 & 411 & 479 &282 \\
    \bottomrule
  \end{tabular}
\end{table}

\begin{table}[H]
  \caption{Optimally tuned hyperparameters for SGD as found by random search}
  \label{tab:sgd_optimal_hyperparameters}
  \centering
  \begin{tabular}{llllll}
    \midrule
    Hyperparameter     & MNIST Classifier & Resnet-18 & Lenet-5 & Senet-18 & LSTM \\
    \midrule
    Initial Learning Rate & 0.0292 & 0.0174 & 0.0229 & 0.100 & 0.0142\\
    Weight Decay & 1e-5 & 1e-8 & 1e-5 & 1e-5 &0.1 \\
    Scheduler $T_{max}$ & 273 & 126 & 93 & 53 &3 \\
    \bottomrule
  \end{tabular}
\end{table}

\begin{table}[H]
  \caption{Optimally tuned hyperparameters for AdaGrad as found by random search}
  \label{tab:adagrad_optimal_hyperparameters}
  \centering
  \begin{tabular}{llllll}
    \midrule
    Hyperparameter     & MNIST Classifier & Resnet-18 & Lenet-5 & Senet-18 & LSTM \\
    \midrule
    Learning Rate & 0.0588 & 0.665 & 0.00549 & 0.00365 & 0.181 \\
    Weight Decay & 1e-6 & 1e-3 & 1e-6 & 0.01 & 0.422 \\
    Learning Rate Decay & 1e-5 & 0.0603 & 1e-6 & 1e-4 & 0.198 \\
    \bottomrule
  \end{tabular}
\end{table}

\begin{table}[H]
  \caption{Optimally tuned hyperparameters for RMSProp as found by random search}
  \label{tab:rmsprop_optimal_hyperparameters}
  \centering
  \begin{tabular}{llllll}
    \midrule
    Hyperparameter     & MNIST Classifier & Resnet-18 & Lenet-5 & Senet-18 & LSTM \\
    \midrule
    Learning Rate & 0.00279 & 0.00672 & 0.002 & 0.000162 & 0.00827 \\
    Weight Decay & 1e-7 & 1e-8 & 1e-6 & 1e-7 & 0.052 \\
    $\alpha$ & 0.9427 & 0.302 & 0.922 & 0.927 & 0.829 \\
    \bottomrule
  \end{tabular}
\end{table}

\begin{table}[H]
  \caption{Optimally tuned hyperparameters for ECCO-DNN as found by random search}
  \label{tab:ec_optimal_hyperparameters}
  \centering
  \begin{tabular}{llllll}
    \midrule
    Hyperparameter     & MNIST Classifier & Resnet-18 & Lenet-5 & Senet-18 & LSTM \\
    \midrule
    LTE Tolerance ($\eta$) & 20.51 & 31.12 & 85.11 & 0.0328 & 20.71 \\
    \bottomrule
  \end{tabular}
\end{table}

The MSE of the LSTM using the optimally tuned comparison optimizers (Adam, SGD, RMSProp, and AdaGrad) is listed in Table \ref{tab:lstm_tuned_comparison}. We compare this to the average MSE from the random-search for ECCO-DNN, which demonstrates comparable performance to the tuned optimizers without any hyperparameter tuning.

\begin{table}[H]
  \caption{Comparison of MSE of LSTM using Average ECCO-DNN Values and Tuned optimizers}
\label{tab:lstm_tuned_comparison}
  \centering
  \resizebox{\columnwidth}{!}{%
  \begin{tabular}{lllll}
    \midrule
     Average ECCO-DNN Value & Tuned Adam & Tuned SGD & Tuned RMSProp & Tuned AdaGrad \\
    \midrule
     0.0247 & 0.0244 & 0.0244 & 0.0244 & 0.0244 \\
    \bottomrule
  \end{tabular}%
  }
\end{table}

\subsection{Loss Curve for Training Resnet-18 on CIFAR-10}
\label{sec:training_loss_curve_resnet_18}
The training loss for training Resnet-18 using ECCO-DNN and different settings of the comparison optimizers is illustrated in Figure \ref{fig:training_loss_resnet_18}. ECCO-DNN achieves a similar trajectory to that of optimally-tuned comparison optimizers and avoids the trajectory of the comparison optimizers with sub-optimal hyperparameter selections. 

 \begin{figure}
     \centering
     \includegraphics[width=0.9\columnwidth]{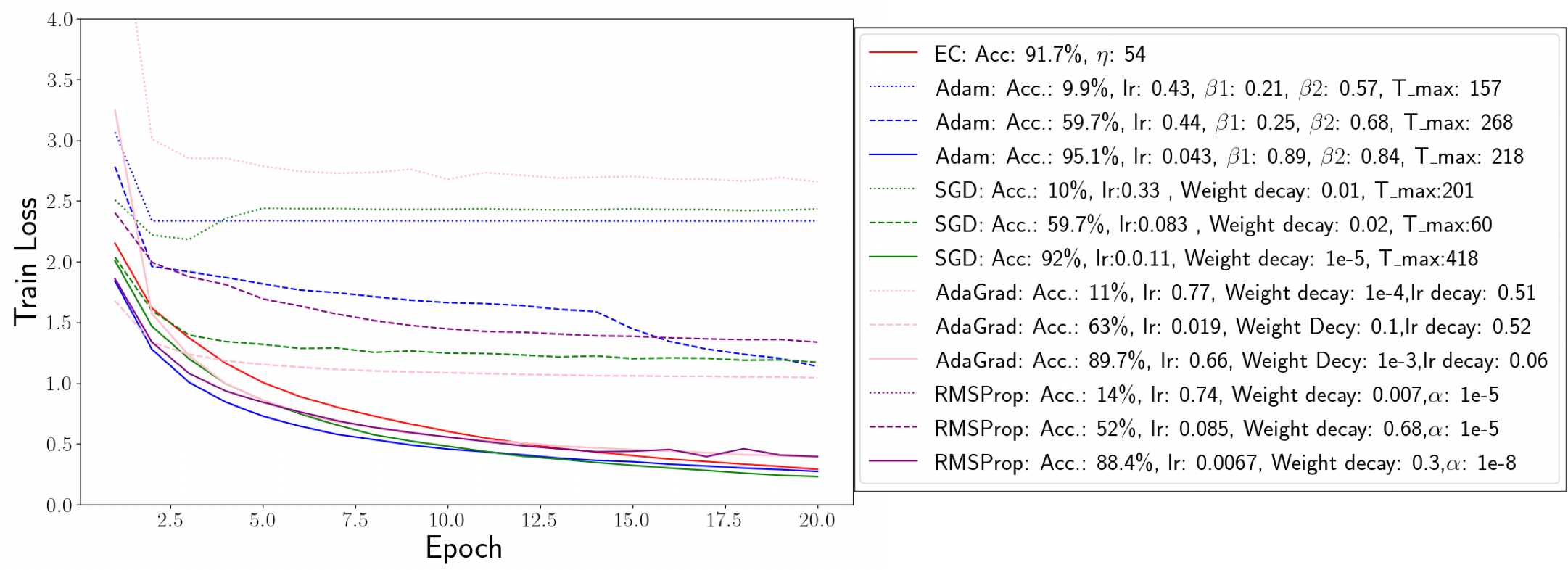}
     \caption{Training Loss of Resnet18 over 20 epochs using multiple settings for comparison optimizers}
     \label{fig:training_loss_resnet_18}
 \end{figure}

\subsection{Normalized Per-Epoch Wallclock Times}

The normalized per-epoch run-time for the experiments are provided in Table \ref{tab:wallclock_times}. ECCO-DNN's insensitivity to hyperparameter values comes at the cost of a slower per-epoch run time. 

\begin{table}[H]
  \caption{Normalized Per-Epoch Wallclock Time for Experiments}
  \centering
  \begin{tabular}{llllll}
    \toprule
      & ECCO-DNN & Adam & SGD & RMSProp & AdaGrad \\
    \midrule
     4-Layer DNN (MNIST) & 1 & 0.85 & 0.84 & 0.85 & 0.83 \\ \midrule
     Resnet18 (CIFAR-10) & 1 & 0.84 & 0.83 & 0.84 & 0.83 \\ \midrule
     Lenet5 (CIFAR-10) & 1 & 0.88 & 0.86 & 0.87 & 0.87 \\ \midrule
     SeNet-18 (CIFAR-10) & 1 & 1.01 & 0.97 & 0.96 & 0.99 \\ \midrule
     LSTM (Power Consumption) & 1 & 0.92 & 0.87 & 0.93 & 0.95 \\
    \bottomrule
  \end{tabular}%
\label{tab:wallclock_times}
\end{table}

\subsection{Training Resnet-50 and Resnet-101 models on CIFAR-100 dataset}
\label{sec:trainig_cifar_100}
We demonstrate ECCO-DNN's scalability to larger datasets and models by training the Resnet-50 and Resnet-101 models on the CIFAR-100 dataset for 100 and 80 epochs, respectively, using ECCO-DNN.  Results in Table \ref{tab:cifar-100-ecco} show a mean test accuracy of 68.1\% and 68.4\% for Resnet-50 and Resnet-101 across 20 random $\eta$ values, with small performance variance (1.88\% and 1.95\%). ECCO-DNN's mean accuracy is within 2\% of optimally tuned optimizers as reported in \cite{shahadat2023enhancing} (Resnet-50) and \cite{dun2022resist} (Resnet-101). Note, our models require no transfer learning or hyperparameter tuning modifications for ECCO-DNN training.

\begin{table}[h]
\centering
  \caption{Mean and Std. classification accuracy of training Resnet-50 and Resnet-101 on CIFAR-100 using ECCO-DNN over 100 and 80 epochs, respectively}

  \begin{tabular}{lll}
    \toprule
     Model & Mean & Std \\
    \midrule
    Resnet-50 & 68.1 & 1.88 \\ \midrule
    Resnet-100 & 68.7 & 1.43 \\ 
    \bottomrule
  \end{tabular}%
\label{tab:cifar-100-ecco}
\end{table}


\end{document}